\documentclass{article}

\usepackage[main, final]{neurips_2026}
\usepackage{mathtools}
\makeatletter
\newif\ifcamera
\@ifpackagewith{neurips_2026}{final}{\cameratrue}{}
\@ifpackagewith{neurips_2026}{preprint}{\cameratrue}{}
\makeatother

\usepackage[utf8]{inputenc} 
\usepackage[T1]{fontenc}    
\usepackage{hyperref}       
\usepackage{url}            
\usepackage{booktabs}       
\usepackage{amsfonts}       
\usepackage{nicefrac}       
\usepackage{microtype}      
\usepackage{xcolor}         
\hypersetup{colorlinks=true,linkcolor=blue,citecolor=blue,urlcolor=blue}
\usepackage{enumitem} 
\usepackage{amssymb} 
\usepackage{amsmath}
\usepackage{amsthm}
\usepackage{mathrsfs}
\newtheorem{theorem}{Theorem}[section]
\newtheorem{lemma}[theorem]{Lemma}

\theoremstyle{definition}
\newtheorem{definition}[theorem]{Definition}

\theoremstyle{remark}
\newtheorem{remark}[theorem]{Remark}

\newcommand{\eps}{\varepsilon}
\newcommand{\opt}{\mathrm{opt}}

\newcommand{\err}{\mathrm{err}}

\title{A Near-optimal SQ Lower Bound for Smoothed Agnostic Learning of Boolean Halfspaces}

\author{%
  Tim Sinen\\
  University of Bonn\\
  Bonn, Germany \\
  \texttt{sinen@uni-bonn.de} \\
}

\begin{document}

\maketitle

\begin{abstract}
  We study the complexity of smoothed agnostic learning of halfspaces on $\{\pm 1\}^n$ under uniform marginals in the model of~\citet{KM25} where each input coordinate is independently flipped with probability $\sigma \in (0, {1}/{2})$. We show that $L^1$ polynomial regression achieves runtime and sample complexity $\tilde{O}(n^{O(\log(1/\varepsilon)/\sigma)})$, and prove a nearly matching Statistical Query complexity lower bound of $n^{\Omega(\log(1+\sigma/\varepsilon^2)/\sigma)}$ in the parameter range of Theorem~\ref{cor:smoothed-agnostic-sq}. This complements the recent work of~\citet{DK26}, which established analogous bounds in the continuous setting under Gaussian marginals.
\end{abstract}

\section{Introduction}
Agnostic learning, introduced by~\citet{Hau92} and~\citet{KSS94}, is one of the most general frameworks for supervised classification. The learner receives labeled examples $(X,Y)$ drawn i.i.d.\ from an arbitrary joint distribution $\mathcal{D}$ over $\mathcal{X} \times \{\pm 1\}$, where $\mathcal{X}$ is a feature space, and must output a hypothesis $h \colon \mathcal{X} \to \{\pm 1\}$ such that, with probability at least $1 - \delta$, the misclassification error satisfies $\Pr_{(X,Y) \sim \mathcal{D}}[h(X) \neq Y] \leq \opt_\mathcal{C} + \varepsilon$, where $\opt_\mathcal{C} = \inf_{f \in \mathcal{C}} \Pr_{(X,Y) \sim \mathcal{D}}[f(X) \neq Y]$ is the error of the best classifier in the target class $\mathcal{C}.$ Importantly, no assumption is made on the relationship between labels and features. Typically, the computational complexity of a learning algorithm is measured in terms of its running time and sample complexity.

It is well known that without any distributional assumptions, even for the fundamental class of parities or halfspaces, agnostic learning is computationally intractable~\citep{GR09, FGKP09, Daniely16}. Under the assumption of structured marginals, positive results are known. Under Gaussian or, more broadly, log-concave marginals, the $L^1$ polynomial regression algorithm of~\citet{KKMS08} provides
the best known agnostic learners for halfspaces
\citep{DGJSV09, DKN10, DKKTZ21}, but the resulting complexity remains exponential in $1/\eps$, with strong evidence that improving it would require either profound new techniques or would have consequences
for other learning problems that are considered hard
\citep{KK14, DKPZ21, DKR23, Tie23}.

One approach to bypass this hardness is the framework of \textit{smoothed agnostic learning}, recently introduced by~\citet{CKKMS24}, which relaxes the benchmark against which the learner must compete. Informally, in this model with domain $\mathbb{R}^n$, inspired by the smoothed analysis paradigm of~\citet{ST04}, the learner’s performance is measured relative to $\opt_{\mathcal{C}, \sigma} = \inf_{f \in \mathcal{C}} \Pr_{(X,Y) \sim \mathcal{D}}[f(T_\sigma X) \neq Y]$ rather than $\mathrm{opt}_\mathcal{C}$, where \(T_\sigma X\) denotes an additively perturbed version of \(X\) parameterized by $\sigma > 0$, effectively smoothing the instance space.

 \citet{KM25} ported this framework to the Boolean hypercube $\{\pm 1\}^n$, where additive perturbations are not well-defined, by replacing them with independent random bit flips of rate $\sigma$. We introduce their framework in more detail.

\begin{definition}[Smoothed agnostic learning on the Boolean hypercube]
    Fix $\eps>0,\sigma\in(0, 1/2)$ and $\delta\in(0,1)$.  Let $\mathcal C\subseteq\big\{\{\pm1\}^n\to\{\pm1\}\big\}$ be a set of Boolean concepts, and let $\mathcal{D}$ be a distribution over $\{\pm1\}^n\times\{\pm1\}$.  We say that an algorithm $A$ \emph{learns} $\mathcal{C}$ in the $\sigma$-smoothed setting if, given i.i.d.\ samples from $\mathcal{D}$, it outputs a hypothesis $h:\{\pm1\}^n\to\{\pm1\}$ such that, with probability at least $1-\delta$,  
\(
\Pr_{(X,Y)\sim\mathcal{D}}\bigl[h(X)\neq Y\bigr]
\le
\mathrm{opt}_{\mathcal{C},\sigma} +\eps,
\)
where
\[
\mathrm{opt}_{\mathcal{C},\sigma}:=
\inf_{f\in\mathcal{C}}\;
\mathbb{E}_{Z\sim B(\sigma)^n}\left [\Pr_{(X,Y)\sim\mathcal{D}}\bigl[
  f(X \odot Z) \neq Y\bigr]
\right].
\]
Here, $Z\sim B(\sigma)^n$ means each coordinate of $Z\in\{\pm 1\}^n$ is independently $-1$ (a flip) with probability $\sigma$, and $x\odot z$ denotes the Hadamard product.
\end{definition}

Note that in this definition, for $\sigma \to 0$ we recover the classical worst-case agnostic learning framework. For $\sigma > 0$, the idea is that the perturbation effectively limits the expressive power of the hypothesis class $\mathcal{C}$ that the learner must compete against, ideally leading to efficient learners with better sample complexity and runtime bounds than known in the worst case.
Also, we do not require $h$ to be contained in $\mathcal{C}$, so it is enough to provide an \textit{improper} learner. 

\paragraph{Boolean Halfspaces.}
    A Boolean halfspace is a function $f \colon \{\pm 1\}^n \to \{\pm 1\}$ of the form $f(x) = \mathrm{sign}\left(\sum_{i=1}^n w_i x_i - \theta\right)$, where $w \in \mathbb{R}^n$ and $\theta \in \mathbb{R}$. The majority function $\mathrm{Maj}_n(x) = \mathrm{sign}\left(\sum_{i=1}^n x_i\right)$ is the halfspace corresponding to uniform weights and zero threshold.
    
    Both in the continuous setting with additive Gaussian perturbations and on the Boolean hypercube with bit-flip noise, halfspaces become efficiently learnable under a wide class of $X$-marginals \(\mathcal D_X\), a phenomenon with no known analogue in the worst-case setting. In more detail, the current best sample-complexity upper bounds for smoothed agnostic learning of halfspaces under subexponential marginals are $n^{\tilde{O}(1/\sigma^2)\log(1/\varepsilon)}$ in the continuous setting \citep{KW25} and $\tilde{O}\!\left(n^{\mathrm{poly}(1/(\sigma\varepsilon))}\right)$ on the Boolean hypercube \citep{KM25}. Notably, both bounds essentially rely on $L^1$ polynomial regression on the smoothed target.
    
    This raises a natural and fundamental question: \begin{center}
        \textit{Is there a way to improve upon the complexity of $L^1$ polynomial regression for learning halfspaces in the smoothed agnostic setting?}
    \end{center}

\paragraph{SQ Lower Bounds.} A powerful framework for studying such questions is the Statistical Query (SQ) model of~\citet{Kea98}, where, instead of receiving samples, a learner obtains approximate answers to statistical queries from an oracle.
\begin{definition}[STAT oracle]
 For a tolerance parameter $\tau > 0$, $\operatorname{STAT}(\tau)$ is the oracle that, given any 
 query function $h: \mathcal{X} \to [-1, 1]$, returns a value
 \[
 v \in \left[\mathbb{E}_{X \sim \mathcal{D}}[h(X)] - \tau,\, \mathbb{E}_{X \sim \mathcal{D}}[h(X)] + \tau\right].
 \]
\end{definition}
Lower bounds in this model are typically stated as a dichotomy: any SQ algorithm must either use 
$n^{\omega(1)}$ queries to $ \operatorname{STAT}(\tau)$ or require a query of tolerance $\tau = n^{-\omega(1)}$. These bounds are of unconditional information-theoretical nature. Crucially, a great number of known algorithmic techniques in machine learning can be captured in the SQ model --- including gradient descent, spectral methods, and most approaches to convex optimization (for a survey, see~\citet{Reyzin20}). This makes SQ lower bounds important evidence of computational hardness.

In this work, we ask for SQ lower bounds in the case of smoothed agnostic learning of Boolean halfspaces.

We note that the open question is not whether in this setting a fully polynomial SQ learner exists, as this is already ruled out by worst-case hardness. Indeed, a union bound over the coordinates of $Z$ gives $\opt_{\mathcal{C},\sigma} \le \opt_{\mathcal{C}} + n\sigma$, so running a smoothed learner with $\sigma = \varepsilon/(2n)$ yields an ordinary agnostic learner with error $\opt_{\mathcal{C}}+\varepsilon$ on the same distribution. Such a learner would use $\mathrm{poly}(n, 1/\sigma, 1/\varepsilon)$ queries and inverse polynomial tolerance, which contradicts known SQ lower bounds for agnostic learning \citep[cf.][]{KKMS08, DFTWW15}. This argument, however, breaks down for $\sigma = \Theta(1)$ and says nothing about the \emph{fine-grained complexity}: is $L^1$ regression essentially optimal for smoothed agnostic learning on the Boolean hypercube as well?

\paragraph{Our contribution.} We give a positive answer to the last question on the Boolean hypercube.

(i) We prove a Statistical Query lower bound for smoothed agnostic learning of Boolean halfspaces on $\{\pm 1\}^n$ under the uniform distribution, which falls within the strictly subexponential class studied by~\citet{KM25}. This matches the upper bound achievable by $L^1$ polynomial regression up to a $\log(1/\sigma)/\sigma$ summand in the exponent.

(ii) As the central technical ingredient, we establish an $L^1$ approximation lower bound for the smoothed majority function $T_{1-2\sigma}\mathrm{Maj}_n$ under the uniform measure on $\{\pm 1\}^n$. The proof constructs an explicit dual witness which combines a symmetric trigonometric kernel with weighted Laguerre integration. This construction may be of independent interest for other SQ lower bounds and approximation problems on the Boolean hypercube.

In more detail, we have the following results.

\begin{theorem}[informal version of Theorem~\ref{cor:smoothed-agnostic-sq}]
Fix a constant smoothing parameter $\sigma \in (0,0.499]$ and let
$1/\operatorname{poly}(n)\le\varepsilon\le\varepsilon_0$ for a sufficiently
small absolute constant $\varepsilon_0>0$. Assume $n$ is sufficiently large.

Any SQ algorithm that, given access to an unknown distribution on
$\{\pm1\}^n\times\{\pm1\}$ with uniform $X$-marginal, outputs with
probability at least $2/3$ a hypothesis with excess error at most
$\varepsilon$ against the best halfspace in the $\sigma$-smoothed setting
must either make a STAT query of tolerance at most
$n^{-\Omega(\log(1+\sigma/\varepsilon^2)/\sigma)}$
or make $2^{\Omega(\sqrt n)}$ queries.
\end{theorem}

Notice that in the regime $\varepsilon \le \sigma$, this gives an SQ lower bound of
\(
    n^{\Omega\!\left({\log(1/\varepsilon)}/{\sigma}\right)}
.\)
As $\sigma \to 0$, we obtain a lower bound of
\(
    n^{\Omega(1/\eps^2)}
\) whenever $\eps = \Omega(n^{-1/4 + \delta})$ for any $\delta > 0$.
This recovers the optimal\footnote{E.g., the optimal exponent $1/\eps^2$ is achievable via the sparse parity reduction which we also discuss below. We also expect that the SQ lower bound of~\citet{DKPZ21}, which is stronger in terms of query complexity, can be transferred to the Boolean hypercube under uniform marginals via the block embedding we also discuss in the following.} SQ lower bound for ordinary agnostic learning of Boolean
halfspaces in the worst-case setting under a uniform marginal in that particular regime. On the upper bound side, we have the following result, which establishes that the SQ lower bound in Theorem 1.3 is essentially tight.
    \begin{theorem}\label{thm:smoothed-upperbound}
Let $\mathcal C$ be any class of Boolean functions on $\{\pm1\}^n$, and let $\varepsilon>0$, $\sigma\in(0,1/2)$ and $\delta\in(0,1)$. There is an algorithm that, given samples from any distribution $\mathcal D$ on $\{\pm1\}^n\times\{\pm1\}$ with uniform $X$-marginal, has sample complexity and running time $\mathrm{poly}(n^{\log(1/\varepsilon)/\sigma},\log(1/\delta))$ and outputs a Boolean hypothesis $h$ with $\err(h)\le\opt_{\mathcal C,\sigma}+\varepsilon$ with probability at least $1-\delta$.
\end{theorem}   
\begin{remark}\label{rem:product}
Theorem~\ref{thm:smoothed-upperbound} still holds if the uniform $X$-marginal is replaced by an arbitrary product distribution $\mu=\mu_1\otimes\cdots\otimes\mu_n$ on $\{\pm1\}^n$.
\end{remark}

\section{Related Work}

The SQ framework was extended to general computational problems by~\citet{FGRVX17}, who introduced the notion of statistical dimension
and showed that lower bounds on it directly translate into query
complexity lower bounds.

In particular, it has been remarkably successful for understanding the limits of distribution-specific agnostic learning. In their seminal work, \citet{DFTWW15} introduced the notion of \emph{approximate resilience}, equivalently $L^1$-inapproximability by low-degree polynomials, and showed that it essentially characterizes the SQ complexity of agnostic learning under the uniform distribution on $\{\pm 1\}^n$. This gives a general framework for deriving SQ lower bounds for learning Boolean concepts, and it is the most relevant precursor for our witness construction. The approximate resilience framework was subsequently extended by~\citet{HSSV22} to obtain near-optimal SQ lower bounds for agnostically learning intersections of halfspaces under Gaussian marginals.

Under Gaussian marginals, the main technique is the moment-matching paradigm of~\citet{DKS17}, in its supervised form~\citep{DKZ20, DKPZ21}. There, one constructs a pair of one-dimensional distributions, one for each label, whose first $m$ moments are the same as those of a reference distribution while being far apart in total variation distance, and then embeds this structure along hidden directions in high dimensions. It has since yielded strong SQ lower bounds across a range of learning problems~\citep{DKZ20, DKPZ21, DK22}. In particular, \citet{DKPZ21} proved that $L^1$ polynomial regression is essentially optimal among SQ algorithms for agnostically learning halfspaces under Gaussian marginals, and more recently,~\citet{DK26} established the first SQ lower bound for smoothed agnostic learning in the continuous subgaussian setting, as we already briefly discussed above. This provides evidence that even with smoothing, $L^1$ regression cannot be substantially improved, which is a negative answer to our question in the continuous setting. Our high-level proof idea follows this paradigm, but the construction of~\citet{DK26} does not directly transfer to the Boolean hypercube, as we discuss in Section~4.

The connections between polynomial approximation and hardness in learning theory are not coincidental: at the heart of all such lower bounds lies the question 
of how well a degree-$m$ polynomial can approximate a target function in a weighted $L^p$ norm. 
This theme appears across several areas. Most relevantly, \citet{BT13} and~\citet{BKT18} developed systematic dual-polynomial methods for approximate-degree 
lower bounds, whose central technical task --- bounding the $L^1$ norm of a dual witness 
constructed from orthogonal polynomials --- is structurally analogous to our normalization 
problem of bounding the $L^\infty$ norm of the sine-Laguerre witness (Lemma~\ref{lem:linfty}).

\paragraph{Weaker and transferred lower bounds.}
A weaker bound already follows from sparse parities, in the spirit of~\citet{KKMS08}
and~\citet[Theorem~E.7]{CKKMS24}. For odd $k$ and $|S|=k$, draw $x$ uniformly at random and assign it label $\chi_S(x)$. One of the halfspaces $\pm\mathrm{Maj}_S$ then has smoothed correlation
$\rho^k|\widehat{\mathrm{Maj}_S}(S)|=\Theta(\rho^k/\sqrt k)$ with the labels, so any
smoothed agnostic learner with $\varepsilon\le c\rho^k/\sqrt k$ must output a hypothesis with correlation to $\chi_S$ of $\Omega(\rho^k/\sqrt{k})$. With $\Lambda:=\log(1+\sigma/\varepsilon^2)/\sigma$
and $k=\Theta(\Lambda)$, SQ hardness of learning parities~\citep{BFJ+94} gives
$Q/\tau^2\ge n^{\Omega(\Lambda)}$, which is the same exponent as in Theorem~\ref{cor:smoothed-agnostic-sq}. However, since there are only $\binom nk=n^{O(\Lambda)}$
candidate parities, they can be learned with $n^{O(\Lambda)}$ queries of constant
tolerance, so this family cannot yield the much stronger query-complexity bound of Theorem~\ref{cor:smoothed-agnostic-sq}, which requires exponentially many nearly
uncorrelated distributions.

Alternatively, Gaussian SQ lower bounds can be transferred to the hypercube by coupling
each Gaussian coordinate to a block of uniform bits  and
approximating Gaussian halfspaces by halfspaces on the block sums using Wasserstein
bounds, e.g.~\citep{Rio09}. We expect that applying this to~\citet{DK26} gives an
alternative route to a bound of similar form. Such a transfer,
however, inherits the exponent of~\citet{DK26}, which is of order
$\log(1/\varepsilon)+1/(\sigma+\varepsilon^2)$ after translating Gaussian to Boolean
smoothing.\footnote{In~\citet{DK26}, the exponent is $\log(1/\varepsilon)+(s+\varepsilon)^{-2}$
for additive Gaussian noise of standard deviation $s$. Flipping each bit with
probability $\sigma$ acts on a block sum like Gaussian noise of standard deviation
$\Theta(\sqrt\sigma)$, so $s=\Theta(\sqrt\sigma)$.}
Up to constants, the exponent of Theorem~\ref{cor:smoothed-agnostic-sq} is at least as
large, and exceeds it by a factor of order $\min\{\log(1/\varepsilon),1/\sigma\}$ whenever
$\varepsilon\le\sigma$. To our knowledge, such a multiplicative dependence is not known in
the Gaussian setting. Moreover, the transfer does not yield a bounded approximation-resilient
witness.

\section{Preliminaries}
We briefly review the most relevant definitions, facts and notation needed for a high-level overview of our techniques.

\paragraph{Notation.} For $n \ge 1$, we set $[n] = \{1,\dots,n\}$. For a measure $\mu$ and $p \ge 1$, we write $\|f\|_{L^p(\mu)}$ for the usual $L^p$ norm, $\|f\|_{L^\infty(\mu)}$ for the essential supremum, and $\langle f, g\rangle_{L^2(\mu)} = \int fg \, d\mu$. When $\mu$ is clear from context we write $\|f\|_p$ and $\langle f, g\rangle$. Probability distributions are denoted by $\mathcal{D}$, with subscripts and decorations (e.g.\ $\mathcal{D}_u$, $\mathcal{D}_0$) for indexed families; we write $\mathcal{D}_X$ for the $X$-marginal of $\mathcal{D}$, and use script $\mathscr{D}$ for classes of distributions. We write $U_n$ for the uniform probability measure on the Boolean hypercube $\{\pm 1\}^n$, so that $\|f\|_{L^p(U_n)} = \mathbb{E}_{X \sim U_n}\left[|f(X)|^p\right]^{1/p}$. Finally, $\mathcal{H}_n$ denotes the class of halfspaces on $\{\pm 1\}^n$, and for $m \ge 0$, $\mathcal{P}_{\le m}$ is the space of multilinear polynomials on $\{\pm 1\}^n$ of degree at most $m$.

\paragraph{Fourier--Walsh Analysis.} Over the Boolean hypercube $\{\pm 1\}^n$, for $S \subseteq [n]$, we set $\chi_S(x) = \prod_{i \in S} x_i$. The family $\{\chi_S\}_{S \subseteq [n]}$ is an orthonormal basis of $L^2(U_n)$, that is, every $f: \{\pm1\}^n\to \mathbb{R}$ admits a unique expansion\[f(x) = \sum_{S \subseteq [n]} \hat{f}(S) \chi_S(x)\] where $\hat{f}(S) = \mathbb{E}_{X \sim U_n}[f(X)\chi_S(X)].$ The degree of the expansion is $\deg(f) = \max\{|S| \mid \hat{f}(S) \neq 0\}$. Parseval's identity gives $\|f\|_{2}^2 = \sum_S \hat{f}(S)^2$. A Boolean function $f:\{\pm 1\}^n\to\mathbb{R}$ is called \emph{symmetric} if 
$f(x_{\pi(1)},\ldots,x_{\pi(n)}) = f(x_1,\ldots,x_n)$ for every permutation $\pi \in S_n$. 
For such functions, it is convenient to organize the Fourier basis by level. The \textit{symmetric Fourier mode} of level $k$ is 
\[ \Psi_{k,n}(x) := \binom{n}{k}^{-1/2}\sum_{|S|=k}\chi_S(x),\quad 0 \le k \le n
\] and the family $\{\Psi_{k,n}\}_{k=0}^n$ is orthonormal in $L^2(U_n).$

\paragraph{Noise operator.} For $\rho \in [0,1]$, the noise operator is $T_\rho f(x) := \mathbb{E}_{Z}[f(x \odot Z)]$ with $Z \sim B((1-\rho)/2)^n$, i.e., each coordinate of $x$ is flipped independently with probability $(1-\rho)/2$. Equivalently, in the Fourier--Walsh basis, $T_\rho f = \sum_S \rho^{|S|} \hat{f}(S) \chi_S$. In particular, $\sigma$-smoothing as in Definition~1.1 corresponds to $T_\rho$ with $\rho = 1-2\sigma$.

\paragraph{SDA Framework.}
In following, let $\mathcal{D}$ be a distribution over an input space $\mathcal{X}$.

\begin{definition}[VSTAT Oracle,~{\citealp[Definition 2.3]{FGRVX17}}]
For a sample size parameter $t > 0$, $\operatorname{VSTAT}(t)$ is the oracle that, given any query function $h: \mathcal{X} \to [0,1]$, returns a value $v \in [p - \tau, p + \tau]$, where $p = \mathbb{E}_{X \sim \mathcal{D}}[h(X)]$ and\[\tau = \max\left\{\frac{1}{t}, \sqrt{\frac{p(1-p)}{t}}\right\}.\]    
\end{definition}
We observe that $\operatorname{VSTAT}(t)$ always returns the expectation within $1/\sqrt{t}$, so it is at least as powerful as $\operatorname{STAT}(1/\sqrt{t})$ and at most as powerful as $\operatorname{STAT}(1/t)$. This allows us to convert VSTAT lower bounds to lower bounds for the STAT oracle, after rescaling the query answer.

For two distributions $\mathcal{D}_1, \mathcal{D}_2$ over $\mathcal{X}$, the \emph{pairwise correlation} relative to $\mathcal{D}$ is \[\chi_\mathcal{D}(\mathcal{D}_1, \mathcal{D}_2) = \left|\left\langle \frac{\mathcal{D}_1}{\mathcal{D}} - 1, \frac{\mathcal{D}_2}{\mathcal{D}} - 1 \right\rangle_{\!\mathcal{D}}\right|.\]
 When $\mathcal{D}_1 = \mathcal{D}_2$, this is equivalent to the $\chi^2$-divergence $\chi^2(\mathcal{D}_1 \| \mathcal{D})$. The \emph{average correlation} of a non-empty set of distributions $\mathscr{D}'$ relative to $\mathcal{D}$ is

\[\gamma(\mathscr{D}', \mathcal{D}) = \frac{1}{|\mathscr{D}'|^2} \sum_{\mathcal{D}_1, \mathcal{D}_2 \in \mathscr{D}'} \chi_\mathcal{D}(\mathcal{D}_1, \mathcal{D}_2).\]

\begin{definition}[Statistical dimension of a set of distributions,~{\citealp[Definition 2.5]{FGRVX17}}]For $\bar{\gamma} > 0$, a set of distributions $\mathscr{D}$ over $\mathcal{X}$, and a reference distribution $\mathcal{D}$ over $\mathcal{X}$, the \emph{statistical dimension of $\mathscr{D}$ relative to $\mathcal{D}$ with average correlation $\bar{\gamma}$}, denoted $\operatorname{SDA}(\mathscr{D}, \mathcal{D}, \bar{\gamma})$, is the largest $d$ such that $\gamma(\mathscr{D}', \mathcal{D}) \leq \bar{\gamma}$ for any subset $\mathscr{D}' \subseteq \mathscr{D}$ with $|\mathscr{D}'| \geq |\mathscr{D}|/d$. \end{definition}
Intuitively, $\operatorname{SDA}(\mathscr{D}, \mathcal{D}, \bar\gamma)$ corresponds to the size of a 
``distinguishable'' family relative to $\mathcal{D}$. A larger 
value of $d$ indicates greater difficulty for statistical queries to distinguish members of 
$\mathscr{D}$ from $\mathcal{D}$.

\begin{definition}[Search problem over distributions,~{\citealp[Definition 2.6]{FGRVX17}}]
A \emph{search problem} $P$ over a set of solutions $F$ and a class of distributions $\mathscr{D}$ 
over $\mathcal{X}$ is a map $P: \mathscr{D} \to 2^F$ that assigns to each $\mathcal{D} \in \mathscr{D}$ a non-empty 
set $P_\mathcal{D} \subseteq F$ of \emph{valid solutions}. An algorithm \emph{solves} $P$ if, given access 
to an unknown $\mathcal{D} \in \mathscr{D}$, it outputs some $f \in P_\mathcal{D}$. For $f \in F$, we write 
$P_f := \{\mathcal{D} \in \mathscr{D} \mid f \in P_\mathcal{D}\}$ for the set of distributions for which $f$ is a valid 
solution.
\end{definition}

\begin{definition}
[Statistical dimension of a search problem,~{\citealp[Definition 3.1]{FGRVX17}}]\label{def:sda-search} Let $P$ be a search problem over a set of solutions $F$ and a class of distributions $\mathscr{D}$ over $\mathcal{X}$. 
For $\bar{\gamma} > 0$ and $\eta \in (0,1)$, the \emph{statistical dimension with average correlation $\bar{\gamma}$ and solution set bound $\eta$}, denoted $\operatorname{SDA}(P, \bar{\gamma}, \eta)$, is the largest $d$ such that there exists a reference distribution $\mathcal{D}$ over $\mathcal{X}$ and a finite set $\mathscr{D}_\mathcal{D} \subseteq \mathscr{D}$ with the following property: for every solution $f \in F$, the set $\mathscr{D}_f = \mathscr{D}_\mathcal{D} \setminus P_f$ satisfies $|\mathscr{D}_f| \geq (1 - \eta) \cdot |\mathscr{D}_\mathcal{D}|$ and $\operatorname{SDA}(\mathscr{D}_f, \mathcal{D}, \bar{\gamma}) \geq d$.
\end{definition}
Note that the witnessing family $\mathscr{D}_\mathcal{D}$ is chosen so that no fixed solution $f \in F$ solves more than an
$\eta$-fraction of its distributions, and the remaining unsolved distributions
$\mathscr{D}_f$ still have statistical dimension at least $d$ relative to $\mathcal{D}$. Thus, a
large value of $\operatorname{SDA}(P,\bar\gamma,\eta)$ means that every candidate solution fails on a large subfamily that still has high statistical dimension. This implies a lower bound on any SQ algorithm by the following theorem.

\begin{theorem}[{\citealp[Theorem 3.2]{FGRVX17}}]\label{thm:sq-lb}Let $P$ be a search problem over a set of solutions $F$ and a class of distributions $\mathscr{D}$ over $\mathcal{X}$. For $\bar{\gamma} > 0$ and $\eta \in (0,1)$, let $d = \operatorname{SDA}(P, \bar{\gamma}, \eta)$. Any randomized SQ algorithm that solves $P$ with probability $\alpha > \eta$ requires at least\[\frac{\alpha - \eta}{1 - \eta} \cdot d\]calls to $\operatorname{VSTAT}(1/(3\bar{\gamma}))$.\end{theorem}

\paragraph{The $L^1$ regression algorithm.} The algorithmic baseline for agnostic learning under structured marginals is the $L^1$ polynomial regression algorithm of~\citet{KKMS08}. Given $N$ samples $(x_i, y_i)$ from a distribution $\mathcal{D}$ over $\{\pm 1\}^n \times \{\pm 1\}$, the algorithm solves the linear program 

\[\min_{p \in \mathcal{P}_{\le d}} \frac{1}{N}\sum_{i=1}^N |p(x_i) - y_i|
\]
over polynomials of degree at most $d$, and outputs $h(x) = \operatorname{sign}(p(x) - t)$ for a suitable threshold $t$. If every target $f$ in a concept class $\mathcal{C}$ can be approximated by a polynomial $p_f$ of degree $d$ such that
\[\|f - p_f\|_{L^1(\mathcal{D}_X)} \le \varepsilon,\]
then the algorithm achieves error $\operatorname{opt}_\mathcal{C} + O(\varepsilon)$ with probability at least $1 - \delta$ with runtime and sample complexity of $\operatorname{poly}(n^d,\log(1/\delta))$.
\section{Technical Overview}
In this section, we give an overview of the structure of our proof for Theorem 1.3 and a proof sketch of Theorem 1.4. Throughout, we set $\rho := 1-2\sigma$, and for a distribution $\mathcal{D}$ over $\{\pm 1\}^n \times \{\pm 1\}$ we write $\opt_{\sigma,\mathcal{D}} := \opt_{\mathcal{H}_n,\sigma}$ for the smoothed optimum over Boolean halfspaces (cf.\ Definition~1.1).

Also, we write $\err_\mathcal{D} (h) = \Pr_{(X,Y) \sim \mathcal{D}}[h(X)\neq Y]$ for the true error of a hypothesis $h$. When the distribution 
$\mathcal{D}$ is clear from context, we drop the $\mathcal{D}$ in these definitions.

At a high level, the SQ lower bound has two main ingredients. First, we establish an $L^1$ approximation lower bound for the smoothed majority $T_\rho \mathrm{Maj}_n$. This is done by constructing an explicit dual witness which is bounded in \(L^\infty\), orthogonal to all
low-degree polynomials, and still has nontrivial correlation with
\(T_\rho \mathrm{Maj}_n\). Second, we embed this one-dimensional hard instance
into many different directions \(u \in \{\pm 1\}^n\), hidden through a large
family of labeled distributions \(\{\mathcal D_u\}_{u \in \mathcal{U}}\). That is, for every $u\in\mathcal{U}$ the labels and the smoothed halfspace $T_\rho\mathrm{Maj}_n(u\odot x)$, whose normal vector is parallel to $u$, have detectable correlation under $\mathcal{D}_u$. Learning the halfspace therefore effectively corresponds to finding the hidden direction. We then exploit the standard \textit{many-vs-one} trick used by~\citet{FGRVX17}: distributions encoding different directions have very small pairwise correlation relative to a reference distribution $\mathcal{D}_0$, so no SQ learner can efficiently distinguish them from this reference, which makes finding the hidden direction hard. The statistical-dimension
framework of~\citet{FGRVX17} then converts the precise correlation bounds into an SQ lower bound.

Note that the restriction to odd $n$ in our argument is purely for convenience and does not affect the asymptotic bounds; the case of 
even $n$ reduces to odd $n-1$ by ignoring one coordinate.

\paragraph{Technical comparison with~\citet{DK26}.} In the continuous Gaussian setting,~\citet{DK26} embed a one-dimensional hard instance into a random direction of $\mathbb{R}^n$ via the Non-Gaussian Component Analysis (NGCA) framework, and establish a lower bound via an implicit Hermite witness for the smoothed sign function, combined with concentration arguments to pass from $L^2$ to $L^1$. On the hypercube, a discrete analogue of the embedding step via hidden juntas was developed by~\citet{DKS22} and used by~\citet{DKRS22} to obtain SQ lower bounds for Boolean halfspaces with Massart noise. Our embedding via the sign flips $x\mapsto u\odot x$ is an alternative whose instances are symmetric functions of $u\odot x$ depending on all $n$ coordinates. This avoids the polynomial dimension overhead of the junta construction and yields a slightly better trade-off between the size of the family and the pairwise correlations. The main obstacle is the witness, as the symmetric Fourier modes that diagonalize the noise operator $T_\rho$ are not uniformly bounded, so the Hermite-based argument has no direct counterpart. We instead construct an explicit witness within the class of bounded functions.

\subsection{Construction of the Sine-Laguerre Dual Witness}
The foundation of the hardness result is an \(L^1\) approximation lower bound for the smoothed majority $T_\rho \mathrm{Maj}_n$ against polynomials of degree $\le m$, that is, we want to lower bound
\[ \inf_{\deg(p) \le m} \|T_\rho \mathrm{Maj}_n - p\|_{L^1(U_n)}. \]
 A \emph{dual witness} is a function $\psi : \{\pm 1\}^n \to [-1, 1]$ with $\psi \perp \mathcal{P}_{\le m}$. For any such $\psi$, weak duality gives
\[ \|T_\rho \mathrm{Maj}_n - p\|_{L^1(U_n)} \;\ge\; \langle T_\rho \mathrm{Maj}_n - p, \psi\rangle \;=\; \langle T_\rho \mathrm{Maj}_n, \psi\rangle \]
for every $p \in \mathcal{P}_{\le m}$, so it suffices to construct a witness $\psi$ with large enough correlation with the smoothed majority.

\paragraph {Construction (informal).}
The smoothed majority is a symmetric odd function; it is natural to ask this of the witness as well. Moreover, every Fourier coefficient of the majority of level $2r+1$ has sign $(-1)^{r}$~\citep[cf.][Chapter 5]{ODonnell14}, so we need to match this oscillatory behavior in order to obtain a satisfyingly large correlation. A natural first candidate is the sine kernel $\sin(\theta \sum_{i=1}^n x_i)$ for some scaling $\theta \in \mathbb R$, having Fourier coefficient $
(-1)^r(\sin\theta)^{2r+1}(\cos\theta)^{n-2r-1}$ at level $2r+1$, which indeed gives the desired oscillation; the function is odd, symmetric and bounded by 1 by definition. The only thing missing is the required orthogonality to lower levels. The trick lies in the observation that by the reparameterization
$\theta = \arctan\sqrt{y/n}$, the Fourier coefficient becomes
$(-1)^r n^{-r-1/2} y^{r+1/2}(1+y/n)^{-n/2}$, so for vanishing on the levels
$1 \le 2r+1 \le m$, it is sufficient to find a weight that cancels the moments $y^{1/2}, y^{3/2}, \dots, y^{k-1/2}$ with respect to $(1+y/n)^{-n/2}\,dy$, where $k := \lfloor (m+1)/2 \rfloor$. This is accomplished by the
weight $(-1)^kL^{(1/2)}_k(y)(1+y/n)^{n/2}e^{-y}$, using Laguerre orthogonality with respect to $y^{1/2}e^{-y}\,dy$.
Averaging the resulting functions for indices $k$ and $k+1$ and normalizing gives the \emph{sine-Laguerre witness} $\psi_{k,n}$ (cf.\ Definition~\ref{def:sine-laguerre-witness}). The particular choice of the Laguerre polynomial (in fact, other polynomial families are conceivable for establishing the orthogonality as well) is motivated by the observation that under the substitution $y=s^2$ the averaged Laguerre kernel becomes a weighted Hermite function, whose Fourier
transform is uniformly bounded. This becomes pivotal in verifying that the weighting does not blow up the $L^\infty$ norm of the witness, so that renormalization preserves correlation with $T_\rho \mathrm{Maj}_n$, i.e., we show that the witness resulting from the averaging with the Laguerre weight has norm $O(1)$ in $L^\infty$ (cf.\ Lemma~\ref{lem:linfty}).

Using the explicit Fourier representations for both our witness and the majority, we obtain the following lower bound.

\begin{lemma}[informal version of Lemma \ref{lem:dual-witness}]\label{lem:witness}
Fix $c>0$ and $\eta\in(0,1/2)$. For every $\rho \in (0,1)$, sufficiently large odd $n$ and integer $1\le m\le cn^{1/2-\eta}$, it holds

\[ \inf_{\deg p \le m} \|T_\rho \mathrm{Maj}_n - p\|_{L^1(U_n)} \ge \Omega_{c,\eta}(\rho^{2m+1}/\sqrt m). \]    
\end{lemma} 
It is sharp in the sense that we cannot hope for a better dependence on the exponent of $m$ in the denominator, see the discussion in the last section.

\subsection{Reduction to a Search Problem}
Next, we embed $\psi$ into a hidden-direction family $\{\mathcal D_u\}_{u\in\mathcal U}$. We show that any smoothed agnostic learner with
small excess error solves the resulting hidden-direction search problem for recovering the direction $u\in\mathcal U$ when the SQ learner has access to an unknown $\mathcal{D}_u$ of this family. For this, it is important to recognize that over the Boolean hypercube, the smoothed error of a halfspace \(g\) can be written in
terms of its smoothed correlation with the labels as
\[
\mathbb E_{Z\sim B(\sigma)^n}\Pr_{(X,Y)\sim \mathcal{D}}[g(X\odot Z)\neq Y]
=
\frac12\left(1-\mathbb E_{(X,Y)\sim \mathcal{D}}[Y\,T_{\rho}g(X)]\right).
\tag{S}\]

\begin{lemma}[Hidden-direction family, informal version of Lemma~\ref{lem:hidden-direction}]\label{lem:hd}
Fix $c>0$ and $\eta\in(0,1/2)$. Let $n$ be odd and large enough, $1\le m\le cn^{1/2-\eta}$ be an integer, and $k=\lfloor(m+1)/2\rfloor$. Let $\psi=\psi_{k,n}:\{\pm1\}^n\to[-1,1]$ be the sine-Laguerre witness from
Lemma~4.1. For $u\in \{\pm1\}^n$, define the
$u$-rotated witness $\psi^{(u)}(x):=\psi(u\odot x)$ and the planted
distribution
\[
  \mathcal D_0(x,y):=2^{-(n+1)},\qquad
  \mathcal D_u(x,y):=2^{-(n+1)}\bigl(1 + y\,\psi^{(u)}(x)\bigr)
\]
for $x \in \{\pm 1\}^n$ and $y \in \{\pm 1\}$.
Then there exists $\mathcal U\subseteq \{\pm1\}^n$
of size $|\mathcal U|\ge \exp(\Omega(\sqrt n))$ with
$|\langle u,v\rangle|\le n^{3/4}$ for all distinct $u,v\in\mathcal U$,
such that:
\begin{enumerate}[label=(\roman*)]
  \item For every $h:\{\pm 1\}^n\to[-1,1]$ and every $u\in\mathcal U$,
        \[
          \mathbb E_{(X,Y)\sim \mathcal D_u}[Yh(X)]
          \;=\;\langle h,\psi^{(u)}\rangle_{U_n}.
        \]
  \item For all distinct $u,v\in\mathcal U$,
        \(
          \chi_{\mathcal D_0}(\mathcal D_u, \mathcal D_v) = |\langle\psi^{(u)},\psi^{(v)}\rangle_{U_n}|
          \le n^{-\Omega(m)}.
        \)
\end{enumerate}
\end{lemma}
Lemma~\ref{lem:hd} makes the link between smoothed correlation with the labels and the search problem framework of the SDA bound precise: Let $P$ be
the problem whose instances are the distributions $\mathcal{D}_u$, $u\in\mathcal U$,
and whose valid solutions are hypotheses $h$ satisfying $\mathbb E_{\mathcal{D}_u}[Yh(X)]\ge \varepsilon$.
Then Lemma~\ref{lem:hd}(i) tells us that $h$ solves $\mathcal{D}_u$ if and only if $\langle h, \psi^{(u)}\rangle_{U_n}\ge \varepsilon$. For this search problem, the usual SQ framework gives the following lower bound.

\begin{theorem}[Hidden-direction SQ lower bound, informal version of Theorem~\ref{thm:hidden-direction-SQ}]
Let \(\{\mathcal{D}_u\}_{u\in\mathcal U}\) be the distribution family from Lemma~\ref{lem:hd}, and let $\gamma_n$ be as in Lemma~\ref{lem:hidden-direction}. Assume $\varepsilon>0$, $\varepsilon^2\ge2\gamma_n$, and that every $\mathcal D_u$ admits a valid solution.
Then any randomized SQ algorithm which, given access to an unknown \(\mathcal{D}_u\), outputs with probability at least $2/3$ a
hypothesis \(h:\{\pm1\}^n\to[-1,1]\) satisfying
\[
\mathbb E_{(X,Y)\sim \mathcal{D}_u}[Yh(X)]\ge \varepsilon
\]
must either make
\(
2^{\Omega(\sqrt n)}
\)
STAT queries, or use tolerance
\(
\tau\le n^{-\Omega(m)}.
\)
\end{theorem}
Note that Theorem~4.3 applies to improper learners which return values in $[-1,1]$; Theorem~1.3 likewise extends to such learners if its error guarantee $\text{err}_{\mathcal{D}}(h) \le \text{opt}_{\sigma, \mathcal{D}} + \varepsilon$ is replaced by the correlation guarantee $\mathbb{E}_{\mathcal{D}}[Y h(X)] \ge 2\eps$.

Theorem~4.3 follows from two facts. First, since the $\psi^{(u)}$ are nearly orthogonal by Lemma~\ref{lem:hd}(ii), a Cauchy--Schwarz argument shows that every hypothesis solves at most $O(1/\varepsilon^2)$ of the distributions \(\{\mathcal D_u\}_{u\in\mathcal U}\).

 Second, every sufficiently large subfamily has average pairwise correlation at most
$O(\bar\gamma)$, where $\bar\gamma=n^{-\Theta(m)}$ is the decay rate from
Lemma~\ref{lem:hd}(ii). In particular, the distributions not solved by a fixed
hypothesis still form a family of high statistical dimension, i.e., for $\eta = o(1)$, the SDA framework of~\citet{FGRVX17} gives
\[
\mathrm{SDA}(P,\,O(\bar\gamma),\eta)\;\gtrsim\;|\mathcal U|\bar\gamma,
\]
so since \(|\mathcal U|\ge\exp(\Omega(\sqrt n))\) and \(\bar\gamma\ge n^{-O(m)}\), this is \(2^{\Omega(\sqrt n)}\) in the regime \(m\log n=o(\sqrt n)\). The matching tolerance bound \(\tau\le n^{-\Omega(m)}\) follows from Theorem~3.5 and converting its VSTAT bound to STAT.

We now have all the ingredients to sketch the proof of our main result, which is a corollary of Theorem~4.3.
\begin{proof}[Proof sketch of Theorem~1.3]
By Lemma~4.1, a short calculation shows that we can choose $m = \Theta(\log(1+\sigma/\eps^2)/\sigma)$ such that $\kappa_m := \langle T_\rho \mathrm{Maj}_n,\psi_{k,n}\rangle_{U_n} \ge 4\eps$, where $k=\lfloor(m+1)/2\rfloor$. The choice in Theorem~\ref{cor:smoothed-agnostic-sq} also ensures $\gamma_n\le2\eps^2$. 

The theorem follows by relating smoothed agnostic learning to the hidden-direction search problem. Each distribution $\mathcal{D}_u$ has a uniform $X$-marginal on $\{\pm 1\}^n$; a lower bound against the family $\{\mathcal{D}_u\}_{u\in\mathcal U}$ therefore implies the same lower bound for smoothed agnostic learning under a uniform $X$-marginal. For $u\in\mathcal U$, we write $g_u(x):=\operatorname{sign}(\langle u,x\rangle)$. 

Since $g_u$ is a halfspace, (S) gives $\opt_{\sigma,\mathcal{D}_u} \le \tfrac12\bigl(1-\mathbb{E}_{\mathcal{D}_u}[Y\,T_\rho g_u(X)]\bigr)$. Hence any Boolean hypothesis $h$ with $\err_{\mathcal{D}_u}(h)\le \opt_{\sigma,\mathcal{D}_u}+\varepsilon$ satisfies
\begin{align*}
\mathbb E_{\mathcal{D}_u}[Yh(X)]
&= 1 - 2\,\err_{\mathcal{D}_u}(h) \\
&\ge \mathbb{E}_{\mathcal{D}_u}[Y\,T_\rho g_u(X)] - 2\varepsilon \\
&= \langle T_\rho g_u, \psi^{(u)}\rangle_{U_n} - 2\varepsilon \\
&= \langle T_\rho \mathrm{Maj}_n, \psi\rangle_{U_n} - 2\varepsilon \\
&= \kappa_m - 2\varepsilon \\
&\ge 2\varepsilon,
\end{align*}
where the second equality is Lemma~\ref{lem:hd}(i) applied to $T_\rho g_u$, and the third is the rotation symmetry $x \mapsto u \odot x$.

Thus, the above learner solves the hidden-direction search problem $P$ from
Theorem~4.3, in particular with correlation threshold \(2\varepsilon\). By Theorem~4.3, any such SQ algorithm must either use a query with
tolerance 
\(
\tau\le n^{-\Omega(m)}
\)
or make
\(
2^{\Omega(\sqrt n)}
\)
queries to STAT. Substituting the above choice of \(m\) gives the claimed SQ lower bound.
\end{proof}

\subsection{Proof Sketch of Theorem~1.4}
Finally, we give a proof sketch of Theorem~1.4, given by a direct application of the $L^1$ regression algorithm of~\citet{KKMS08}.
\begin{proof}[Proof sketch]
For any $f \in \mathcal{C}$, let $p_d$ be the 
Fourier truncation of $T_\rho f$ at degree $d$. Parseval's identity and the Cauchy--Schwarz inequality bound the tail by
$\|T_\rho f - p_d\|_{1} \le \rho^{d+1}$. With $d = O(\log(1/\eps)/\sigma)$ this is 
$O(\eps)$, after which the $L^1$ regression algorithm of~\citet{KKMS08} outputs a Boolean hypothesis $h$ with 
$\err(h) \le \opt_{\mathcal{C},\sigma} + \eps$ via a standard thresholding step. 
\end{proof}

\section{Conclusion and Future Directions}

We studied the complexity of smoothed agnostic learning of Boolean halfspaces
for distributions with uniform \(X\)-marginal. On the algorithmic side, the
standard \(L^1\) polynomial regression method gives a learner with sample
complexity and running time
\(
\operatorname{poly}\!\left(n^{\log(1/\varepsilon)/\sigma},\log(1/\delta)\right).
\)
Our main result shows that this dependence is essentially unavoidable for SQ
algorithms in the parameter range of Theorem~\ref{cor:smoothed-agnostic-sq}: Any SQ learner must either use tolerance
\(
\tau \le n^{-\Omega(\log(1+\sigma/\varepsilon^2)/\sigma)}
\)
or make \(2^{\Omega(\sqrt n)}\) queries.

Several questions remain open. First, on the lower bound side, it would be interesting to close the remaining gaps in Theorem~\ref{cor:smoothed-agnostic-sq} compared to the upper bound. For example, a logarithmic gap in the exponent remains when $\sigma=\Theta(\varepsilon^2)$. Closing it via our approach would require a lower bound in Lemma~\ref{lem:witness} that decays more slowly in $m$ than $\rho^{2m+1}/\sqrt m$ for fixed $\rho<1$. An obvious strengthening to \(
\inf_{\deg(p)\le m}\|T_\rho \mathrm{Maj}_n - p\|_{L^1(U_n)} \ge \rho^m
\) fails for $\rho\to1$, since every halfspace on $\{\pm1\}^n$ has a polynomial $L^1$ approximation of degree at most $m$ with error $\tilde O(1/\sqrt m)$, uniformly in $n$~\citep{DGJSV09, DKN10}.

Second, one could ask whether stronger SQ lower bounds hold for more specific Boolean marginals beyond the uniform distribution, or if it is possible to improve the upper
bounds of~\citet{KM25} for general subexponential marginals, where the best known dependence
\(n^{\mathrm{poly}(1/(\sigma\varepsilon))}\) remains far from the
\(n^{O(\log(1/\varepsilon)/\sigma)}\) upper bound of the uniform
case. 

As mentioned in the introduction, we believe that the idea behind the construction of the sine-Laguerre witness might be applicable to obtain SQ lower bounds for other smoothed settings, e.g.\ intersections of halfspaces or ReLUs.

Additionally, motivated by the worst-case agnostic literature on approximate learners \citep{Daniely15, ABL17}, it is natural to ask whether smoothed agnostic learners with relaxed guarantees of the form $O(\text{opt}_\sigma) + \eps$ or $\text{opt}_\sigma^{O(1)} + \eps$ admit faster algorithms.

Finally, an
important direction is to understand whether non-SQ algorithms can bypass the SQ barrier, or whether \(L^1\) polynomial regression is optimal here in a stronger
computational sense. This question is fundamental to agnostic learning more broadly, well beyond the smoothed setting.

\ifcamera
\section*{Acknowledgments}
The author thanks Heiko Röglin for helpful discussions and feedback on early drafts of this work as well as the anonymous reviewers for helping to improve the quality of the presentation of the results. This work was supported by the Lamarr Institute for Machine Learning and Artificial Intelligence.
\fi
{
\small
\bibliography{refs}            
}

\newpage

\appendix

\section{Some Facts about Orthogonal Polynomials}

This appendix collects some standard identities for the
Hermite, Laguerre, and Krawtchouk families used in the deferred proofs of the paper. Many of them can be found in the literature, e.g.~\citep{Szego, KLS}.

Throughout, we write
\begin{equation}
\widehat f(\xi):=(2\pi)^{-1/2}\int_{\mathbb R} f(x)\,e^{-ix\xi}\,dx
\tag{A.1}\label{eq:A.1}
\end{equation}
for the unitary Fourier transform of $f \in L^2(\mathbb{R})$.

\subsection{Hermite Polynomials and Hermite Functions}
The \emph{physicist's} Hermite polynomials are given by the Rodrigues formula as
\begin{equation}
H_k(x) := (-1)^k e^{x^2}\frac{d^k}{dx^k}e^{-x^2},\qquad k\ge 0.
\tag{A.2}\label{eq:A.2}
\end{equation}
The associated normalized Hermite functions
\begin{equation}
\Phi_k(x) := \bigl(2^{k} k!\sqrt{\pi}\bigr)^{-1/2}\, H_k(x)\, e^{-x^2/2}
\tag{A.3}\label{eq:A.5}
\end{equation}
form an orthonormal basis of $L^2(\mathbb{R})$, satisfy the Hermite ODE
\begin{equation}
-\Phi_k''(x) + x^2\Phi_k(x) \;=\; (2k+1)\,\Phi_k(x)
\tag{A.4}\label{eq:A.6}
\end{equation}
and are eigenfunctions of the (unitary) Fourier transform,
\begin{equation}
\widehat{\Phi_k}(\xi) \;=\; (-i)^{k}\,\Phi_k(\xi).
\tag{A.5}\label{eq:A.7}
\end{equation}
Additionally, they obey the ladder identities
\begin{equation}
x\,\Phi_k \;=\; \sqrt{\tfrac{k+1}{2}}\,\Phi_{k+1} + \sqrt{\tfrac{k}{2}}\,\Phi_{k-1}
\tag{A.6}\label{eq:A.8}
\end{equation}
and
\begin{equation}
\Phi_k' \;=\; \sqrt{\tfrac{k}{2}}\,\Phi_{k-1} - \sqrt{\tfrac{k+1}{2}}\,\Phi_{k+1}
\tag{A.7}\label{eq:A.9}
\end{equation}
with the convention \(\Phi_{-1}\equiv 0\).

Finally, we need the following basic Hermite estimates.

\begin{lemma}\label{prop:basic-hermite}
There exists an absolute constant $C>0$ such that for all $n\ge0$,
\begin{equation}
\|\Phi_n'\|_{L^\infty(\mathbb R)}\le C\,(n+1)^{1/4}
\tag{A.8}\label{eq:A.20}
\end{equation}
and
\begin{equation}
\|\Phi_n\|_{L^1(\mathbb R)}\le C\,(n+1)^{1/4}.
\tag{A.9}\label{eq:A.21}
\end{equation}
\end{lemma}

\begin{proof}
The bound \eqref{eq:A.20} is the case $k=1$ of~\cite[Lemma~3.2]{WZ25}.
For \eqref{eq:A.21}, the Cauchy--Schwarz inequality yields
\begin{align*}
\|\Phi_n\|_{L^1(\mathbb R)}
&\le\Bigl(\int_{\mathbb R}\frac{dx}{x^2+n+1}\Bigr)^{1/2}
\bigl((n+1)\|\Phi_n\|_2^2+\|x\Phi_n\|_2^2\bigr)^{1/2}\\
&\overset{\eqref{eq:A.8}}\le\Bigl(\frac{\pi}{(n+1)^{1/2}}\Bigr)^{1/2}\bigl(2(n+1)\bigr)^{1/2}
=\sqrt{2\pi}\,(n+1)^{1/4},
\end{align*}
where we used orthonormality.
\end{proof}

\subsection{Laguerre Polynomials}
For $\alpha > -1$, the generalized Laguerre polynomials $\{L_m^{(\alpha)}\}_{m\geq 0}$ are orthogonal with respect to the weight $x^\alpha e^{-x}$ on $[0,\infty)$ and are given by
\begin{equation}
L_m^{(\alpha)}(x) \;:=\; \sum_{j=0}^{m} (-1)^j \binom{m+\alpha}{m-j}\frac{x^j}{j!},
\tag{A.11}\label{eq:A.11}
\end{equation}
where $\binom{a}{k} = {a(a-1)\cdots(a-k+1)}/{k!}$ for $a\in\mathbb{R}$, $k\in\mathbb{N}_0$.
Moreover, the moment identity
\begin{equation}
\int_0^\infty x^{\alpha'-1}\, e^{-x}\, L_m^{(\alpha)}(x)\,dx \;=\; \binom{\alpha-\alpha'+m}{m}\,\Gamma(\alpha')
\tag{A.12}\label{eq:A.13}
\end{equation}
holds for \(\alpha>-1\) and \(\mathrm{Re}\;\alpha'>0\), where $\Gamma$ denotes the Gamma function. From this, a short computation shows
\begin{equation}
 \int_0^\infty L_m^{(\alpha)}(x)\,x^{r+\alpha}\,e^{-x}\,dx
\;=\; (-1)^{m}\,\binom{r}{m}\,\Gamma(r+\alpha+1)
\tag{A.13}\label{eq:A.14}
\end{equation}
for any integers $m,r\ge 0$, where we use the convention \(\binom{r}{m}=0\) when \(r<m\).

For $m \ge 0$, the Hermite and Laguerre polynomials are linked through the identity
\begin{equation}
H_{2m}(x) \;=\; (-1)^{m}\,2^{2m}\,m!\,L_m^{(-1/2)}(x^2).
\tag{A.14}\label{eq:A.15}
\end{equation}
Combining \eqref{eq:A.15} with the definition of the normalized Hermite function $\Phi_n$, for $n \ge 1$, the packet
\begin{equation}
A_n(x) \;:=\; (-1)^{n}\, x\, L_n^{(-1/2)}(x^2)\, e^{-x^2/2}
\tag{A.15}\label{eq:A.17}
\end{equation}
admits the factorization
\begin{equation}
A_n(x) \;=\; d_n\, x\, \Phi_{2n}(x),
\qquad
d_n \;:=\; \frac{\bigl(2^{2n}(2n)!\sqrt{\pi}\bigr)^{1/2}}{2^{2n}\, n!}.
\tag{A.16}\label{eq:A.18}
\end{equation}
Squaring the definition of $d_n$ gives
\begin{equation}
d_n^{\,2} \;=\; \sqrt{\pi}\,\frac{\binom{2n}{n}}{4^{n}} \;=\; \Theta(n^{-1/2})
\tag{A.17}\label{eq:A.19}
\end{equation}
where we used Stirling's formula $\binom{2n}{n} \sim 4^n / \sqrt{\pi n}$.

This yields the following further estimates.

\begin{lemma}\label{lem:AM}
There is an absolute constant $C>0$ and, for every $j\ge0$, a constant $C_j>0$ such
that for all $M\ge1$,
\begin{align}
&\|\widehat{A_M}\|_\infty\le C, \tag{A.18}\label{eq:H1}\\
&\|s^jA_M\|_1+\|A_M^{(j)}\|_1\le C_j\,M^{(j+1)/2}. \tag{A.19}\label{eq:H2}
\end{align}
\end{lemma}

\begin{proof}
By \eqref{eq:A.18}--\eqref{eq:A.19},
$A_M(s)=d_M\,s\,\Phi_{2M}(s)$ with $|d_M|\lesssim M^{-1/4}$.
Using \eqref{eq:A.7} and $\widehat{sf}=i\widehat f'$, we obtain
\[
\|\widehat{A_M}\|_\infty
=|d_M|\|\Phi_{2M}'\|_\infty
\overset{\mathclap{\eqref{eq:A.20}}}{\lesssim}
M^{-1/4}M^{1/4}=1,
\]
which is \eqref{eq:H1}.

For \eqref{eq:H2}, the ladder identities
\eqref{eq:A.8}--\eqref{eq:A.9} show that each multiplication
by $s$ or differentiation shifts the Hermite index by one,
and all coefficients are in $O_j(\sqrt M)$ throughout the first $j+1$
steps. Hence
\[
\begin{aligned}
\|s^jA_M\|_1+\|A_M^{(j)}\|_1
&\lesssim_j
  |d_M|M^{(j+1)/2}
  \max_{\substack{\ell\ge0\\|\ell-2M|\le j+1}}\|\Phi_\ell\|_1\\
&\overset{\mathclap{\eqref{eq:A.21}}}{\lesssim_j}
  M^{-1/4}M^{(j+1)/2}M^{1/4}
=M^{(j+1)/2},
\end{aligned}
\]
as required.
\end{proof}

\subsection{Krawtchouk Polynomials}

For $0 \le k \le n$ and $r \in \{-1,\,-1+2/n,\,\dots,\,1-2/n,\,1\}$, the normalized binary Krawtchouk polynomial of degree $k$ is
\begin{equation}
\mathcal{K}_k^{(n)}(r)
\;:=\;
\binom{n}{k}^{-1}
\sum_{b=0}^{k} (-1)^{b}
\binom{{n(1-r)}/{2}}{b}
\binom{{n(1+r)}/{2}}{k-b}.
\tag{A.20}\label{eq:A.22}
\end{equation}
We use the convention \(\binom ab=0\) whenever \(b<0\) or \(b>a\) for nonnegative integers \(a,b\).
Writing $d = n(1-r)/2 \in \{0,1,\dots,n\}$ for the Hamming distance, this is
\begin{equation}
\mathcal{K}_k^{(n)}(r)\;=\;\binom{n}{k}^{-1} K_k(d;n),
\qquad
K_k(d;n)\;:=\;\sum_{b=0}^{k}(-1)^{b}\binom{d}{b}\binom{n-d}{k-b}
\tag{A.21}\label{eq:A.23}
\end{equation}
where \(K_k(\,\cdot\,;n)\) is the unnormalized Krawtchouk polynomial.

The Krawtchouk polynomials satisfy the three-term recurrence in the degree variable
\begin{equation}
(n-2d)\,K_k(d;n)\;=\;(k+1)\,K_{k+1}(d;n)\;+\;(n-k+1)\,K_{k-1}(d;n)
\tag{A.22}\label{eq:A.29}
\end{equation}
with the convention $K_{-1} \equiv K_{n+1} \equiv \mathcal{K}_{-1}^{(n)} \equiv \mathcal{K}_{n+1}^{(n)} \equiv 0$. Using $n-2d = nr$ and $(k+1)\binom{n}{k+1} = (n-k)\binom{n}{k}$, $(n-k+1)\binom{n}{k-1} = k\binom{n}{k}$, this is equivalent to the normalized recurrence
\begin{equation}
nr\,\mathcal{K}_k^{(n)}(r)\;=\;(n-k)\,\mathcal{K}_{k+1}^{(n)}(r)\;+\;k\,\mathcal{K}_{k-1}^{(n)}(r).
\tag{A.23}\label{eq:A.30}
\end{equation}

For \(0\le \ell\le n\), \(x,y\in\{\pm1\}^n\), and
\(D:=|\{i\mid x_i\ne y_i\}|\), we have
\begin{equation}
{\binom{n}{\ell}^{-1}}
\sum_{|S|=\ell}\chi_S(x)\,\chi_S(y)
\;=\;
\frac{K_\ell(D;n)}{\binom{n}{\ell}}
\;=\;
\mathcal{K}_\ell^{(n)}\!\left(\frac{\langle x,y\rangle}{n}\right).
\tag{A.24}\label{eq:A.31}
\end{equation}
This follows from the observation that if \(b\) of the \(\ell\) selected coordinates lie in the disagreement
set, then \(\chi_S(x)\chi_S(y)=(-1)^b\), and summing over \(b\) gives the
Krawtchouk polynomial.

Using $\chi_S(w \odot x) = \chi_S(w) \chi_S(x)$ for \(w\in\{\pm1\}^n\), expanding the product, and applying \eqref{eq:A.31} gives the addition formula
\begin{equation}
\mathbb{E}_{x\sim U_n}\!\left[
\Psi_{\ell,n}(u\odot x)\,
\Psi_{\ell',n}(v\odot x)
\right]
\;=\;
\mathcal{K}_{\ell}^{(n)}\!\left(\frac{\langle u,v\rangle}{n}\right)
\mathbf{1}_{\{\ell=\ell'\}}
\tag{A.25}\label{eq:A.32}
\end{equation}
for every \(u,v\in\{\pm1\}^n\).

\section{Omitted Proofs}
In this appendix, we provide the deferred proofs of the theorems and lemmas from the main section.
\subsection{Proof of Lemma 4.1}
For $x\in\{\pm1\}^n$, let $S(x):=\sum_{i=1}^n x_i$, and for $y\ge0$, let
$\theta_n(y):=\arctan\sqrt{y/n}$ and $\phi_{y,n}(x):=\sin(\theta_n(y)S(x))$. By construction,
$\phi_{y,n}$ is symmetric, odd and bounded in absolute value by $1$.

For odd $n$ and $\theta:=\theta_n(y)$, we have $\phi_{y,n} = \operatorname{Im}{e^{i\theta\sum_{j=1}^n x_j}} =\operatorname{Im}\prod_{j=1}^n(\cos\theta+ix_j\sin\theta)$. Expanding the product and plugging in the substitution gives
\begin{align*}
\phi_{y,n}(x)
&=(\cos \theta)^n\sum_{\substack{d\le n\\ d\text{ odd}}}(-1)^{(d-1)/2}(\tan \theta)^d
\binom nd^{1/2}\Psi_{d,n}(x)\\
&=\Bigl(1+\frac yn\Bigr)^{-n/2}\sum_{\substack{d\le n\\ d\text{ odd}}}(-1)^{(d-1)/2}
  \Bigl(\frac yn\Bigr)^{d/2}\binom nd^{1/2}\Psi_{d,n}(x)\\
&=\Bigl(1+\frac yn\Bigr)^{-n/2}\sum_{r=0}^{(n-1)/2}(-1)^r
  \binom n{2r+1}^{1/2}n^{-r-\frac12}\,y^{r+\frac 12}\,\Psi_{2r+1,n}(x).\tag{$\phi$}
\label{eq:phi-expansion}\end{align*}

Weighting $\phi_{y,n}$ against $L(y)\,(1+y/n)^{n/2}e^{-y}\,dy$ therefore multiplies
the level-$(2r+1)$ coefficient by $\int_0^\infty L(y)\,y^{r+1/2}e^{-y}\,dy$, which
vanishes for all small $r$ when $L$ is a Laguerre polynomial, cf.~\eqref{eq:A.14}.
The average of two consecutive witnesses turns out to give better asymptotics than
a single witness. 

\begin{definition}[Sine--Laguerre witness]\label{def:sine-laguerre-witness}
Let $n \ge 1$ be odd, and let $\phi_{y,n}$ be as in~\eqref{eq:phi-expansion}.
For $\ell\ge1$, the \emph{$\ell$-th unnormalized Laguerre witness} is
\[
\Omega_{\ell,n}(x)
:=
(-1)^\ell
\int_0^\infty
L_\ell^{(1/2)}(y)\,
\Bigl(1+\frac yn\Bigr)^{n/2}e^{-y}\,
\phi_{y,n}(x)\,dy,
\]
and for $j\ge 1$, the \emph{two-term averaged sine--Laguerre witness} is
$\widetilde\Omega_{j,n} := \tfrac12(\Omega_{j,n}+\Omega_{j+1,n})$.
Whenever $\widetilde\Omega_{j,n}\not\equiv 0$, the
\emph{normalized sine--Laguerre witness} is
\[
    \psi_{j,n}
    :=
    \frac{\widetilde\Omega_{j,n}}{\|\widetilde\Omega_{j,n}\|_\infty}.
\]
\end{definition}

\begin{lemma} Let \(n\ge1\) be odd, \(m\ge1\), and set
\(k:=\lfloor (m+1)/2\rfloor\). Assume \(k\le (n-1)/2\). Then the normalized
sine--Laguerre witness $\psi_{k,n}$ is a valid dual witness, i.e.,
$\|\psi_{k,n}\|_\infty \le 1$ and $\psi_{k,n} \perp \mathcal{P}_{\le m}.$
\end{lemma}
\begin{proof}
The first part holds by definition. By the parameter shift
$L_{k+1}^{(1/2)}-L_k^{(1/2)}=L_{k+1}^{(-1/2)}$, the kernel of
$\widetilde\Omega_{k,n}$ can be written as $\tfrac12(-1)^{k+1}L_{k+1}^{(-1/2)}(y)$.
Moreover, by \eqref{eq:A.14},
\[
\int_0^\infty L_{k+1}^{(-1/2)}(y)\,y^{r+1/2}e^{-y}\,dy
=(-1)^{k+1}\binom{r+1}{k+1}\Gamma\bigl(r+\tfrac32\bigr),
\]
which vanishes for $r<k$. Inserting~\eqref{eq:phi-expansion} thus gives
\begin{align}
\widetilde\Omega_{k,n}
&=\frac{(-1)^{k+1}}{2}\int_0^\infty L_{k+1}^{(-1/2)}(y)\,
  \Bigl(1+\frac yn\Bigr)^{n/2}e^{-y}\,\phi_{y,n}\,dy
\nonumber\\
&=\frac12\sum_{r=k}^{(n-1)/2}(-1)^r\binom{r+1}{k+1}\Gamma\bigl(r+\tfrac32\bigr)
  \binom n{2r+1}^{1/2}n^{-r-\frac12}\,\Psi_{2r+1,n}. \tag{E}\label{eq:wite}
\end{align}
For $k\le(n-1)/2$, we readily see $\widetilde\Omega_{k,n}\not\equiv0$
as well as $\widetilde\Omega_{k,n}\perp\mathcal P_{\le 2k}\supseteq\mathcal P_{\le m}$.
\end{proof}
\begin{lemma}\label{lem:dual-witness}
Fix \(0<\rho<1\). For every \(c>0\) and \(0<\eta<1/2\), there exist
\(n_0(c,\eta)\) and \(K_{c,\eta}>0\) such that, for all odd
\(n=2N+1\ge n_0(c,\eta)\) and all integers
\(1\le m\le c n^{1/2-\eta}\),
\[
\inf_{\deg p\le m}
\|T_\rho\operatorname{Maj}_n-p\|_{L^1(U_n)}
\ge
K_{c,\eta}\frac{\rho^{2m+1}}{\sqrt m}.
\]
\end{lemma}
\begin{remark} The assumption $m \le cn^{1/2-\eta}$ simplifies the asymptotics here but is plausibly not essential for the lower bound; without it, however, the proof of Lemma~\ref{lem:linfty} would fail.
\end{remark}
\begin{proof}[Proof of Lemma~\ref{lem:dual-witness}]
In Lemma~\ref{lem:linfty} we show $\|\widetilde\Omega_{k,n}\|_\infty=O_{c,\eta}(1)$, so by weak
duality and the previous lemma it suffices to lower-bound
$\langle T_\rho\operatorname{Maj}_n,\widetilde\Omega_{k,n}\rangle$.
By~\citet[Thm.~5.19]{ODonnell14}, the Fourier coefficients of majority vanish on sets of
even size, and for $|S|=j$ odd,
\[
\widehat{\operatorname{Maj}_n}(S)
=(-1)^{(j-1)/2}\,
\frac{\binom{N}{(j-1)/2}}{\binom{n-1}{j-1}}
\cdot\frac{2}{2^n}\binom{n-1}{N}.
\] 
Combining this with \eqref{eq:wite}, in the correlation $\langle T_\rho\operatorname{Maj}_n,\widetilde\Omega_{k,n}\rangle$
all summands per non-vanishing level are nonnegative, so we may keep only those with index $k\le r\le2k-1$.
Since $2k-1\le m\le N$ and $m^2\le n/2$ for large $n$, after some simplification of the terms, it holds
\begin{align*}
\langle T_\rho\operatorname{Maj}_n,\widetilde\Omega_{k,n}\rangle
&\ge\underbrace{\frac{\binom{2N}{N}\sqrt n}{4^N}}_{\mathclap{=\Theta(1)}}\,
   \frac{\rho}{4}\sum_{r=k}^{2k-1}\rho^{2r}\,
   \underbrace{\frac{N^{\underline r}}{n^r}}_{\mathclap{\ge 2^{-r-1}}}\,
   \underbrace{\frac{\Gamma(r+\frac12)}{\Gamma(r+1)}}_{\mathclap{\gtrsim m^{-1/2}}}\,
   \binom{r+1}{k+1}\\
&\gtrsim_{c,\eta}\frac{\rho}{\sqrt m}\sum_{r=k}^{2k-1}\rho^{2r}\,2^{-r}\binom{r+1}{k+1}\\
&\ge\frac{\rho^{2m+1}}{\sqrt m}\sum_{r=k}^{2k-1}2^{-r}\binom rk,
\end{align*}
where $N^{\underline r}:=\prod_{j=0}^{r-1}(N-j).$

Since $2^{-r-1}\binom rk$ is the probability that the $(k+1)$-st head of a fair coin
appears at toss $r+1$, the last sum is twice the probability of at least $k+1$ heads in
$2k$ tosses, that is, $1-4^{-k}\binom{2k}{k}\ge\frac12$. Hence
$\langle T_\rho\operatorname{Maj}_n,\widetilde\Omega_{k,n}\rangle\gtrsim\rho^{2m+1}/\sqrt m$.
\end{proof}

\begin{lemma}\label{lem:linfty}
For every $c>0$ and $0<\eta<1/2$ there is a constant $C_{c,\eta}>0$ such
that for all odd $n$ and all integers $1\le m\le cn^{1/2-\eta}$, with
$k=\lfloor(m+1)/2\rfloor$, it holds
$\|\widetilde\Omega_{k,n}\|_\infty\le C_{c,\eta}$.
\end{lemma}

\begin{proof}

Fix $x\in\{\pm1\}^n$ and put $\zeta:=S(x)/\sqrt n$; recall that $k+1\le2cn^{1/2-\eta}$.
Let $A_{k+1}$ be as in \eqref{eq:A.17}, and set
$w(s):=e^{-s^2/2}(1+s^2/n)^{n/2}$ and $\nu(s):=\sqrt n\arctan(s/\sqrt n)$.
Changing variables $y=s^2$ in~\eqref{eq:wite} gives
\begin{align*}
\widetilde\Omega_{k,n}(x)
&=\frac{(-1)^{k+1}}{2}\int_0^\infty L_{k+1}^{(-1/2)}(y)\,e^{-y}\Bigl(1+\frac yn\Bigr)^{n/2}
\sin\bigl(\theta_n(y)S(x)\bigr)\,dy\\
&=\int_0^\infty A_{k+1}(s)\,w(s)\sin(\zeta\nu(s))\,ds
=\frac1{2i}\underbrace{\int_{\mathbb R}A_{k+1}(s)\,w(s)\,e^{i\zeta\nu(s)}\,ds}_{=:I(\zeta)},
\end{align*}
as $A_{k+1}$ and $\nu$ are odd and $w$ is even.
Since $w(s)\to1$ and $\nu(s)\to s$ as $n\to\infty$, the integral $I(\zeta)$ is
essentially $\sqrt{2\pi}\,\widehat{A_{k+1}}(-\zeta)$, the Fourier transform of a Hermite
function, which is uniformly bounded by~\eqref{eq:H1}. In the following, making this heuristic rigorous yields $|I(\zeta)| \lesssim_{c,\eta} 1$, which establishes the lemma.

Using $\log(1+z)\le z-z^2/(2(1+z))$ for $z\ge0$, we have
\[
w(n^{1/4}t)
\le \exp\!\left(-\frac{t^4}{4(1+t^2/\sqrt n)}\right)
\le e^{(1-t^2)/4}.\label{eq:w-bound}\tag{W}
\]
For $|\zeta|\le n^{1/4}$, from the preceding bound, we obtain
\[
\bigl\|w\,e^{i\zeta(\nu-s)}\bigr\|_2^2=\|w\|_2^2=n^{1/4}\int_{\mathbb R}w(n^{1/4}t)^2\,dt
\le n^{1/4}\int_{\mathbb R}e^{(1-t^2)/2}\,dt\lesssim n^{1/4},
\]
and similarly $\bigl\|\bigl(w\,e^{i\zeta(\nu-s)}\bigr)'\bigr\|_2^2\lesssim n^{-1/4}$. 
By Young's and the Cauchy--Schwarz inequality, it follows that
\begin{align*}
|I(\zeta)|
&=\Bigl|\int_{\mathbb R}\widehat{A_{k+1}}(\xi)\,
\mathcal F\bigl[w\,e^{i\zeta(\nu-s)}\bigr](-\xi-\zeta)\,d\xi\Bigr|
\le\|\widehat{A_{k+1}}\|_\infty\,\bigl\|\mathcal F\bigl[w\,e^{i\zeta(\nu-s)}\bigr]\bigr\|_1\\
&\overset{\eqref{eq:H1}}{\lesssim}
\Bigl(\bigl\|w\,e^{i\zeta(\nu-s)}\bigr\|_2\,
\bigl\|\bigl(w\,e^{i\zeta(\nu-s)}\bigr)'\bigr\|_2\Bigr)^{1/2}
\lesssim 1 .
\end{align*}
For $|\zeta|>n^{1/4}$, we use the principle of non-stationary phase, as $\nu'(s) = (1+s^2/n)^{-1} > 0$ for all $s \in \mathbb R$.  Defining $ Lf:=(f/\nu')'=\big((1+s^2/n)f\big)'$, $L$
maps the Schwartz class $\mathcal{S}(\mathbb{R})$ into itself. Further, as
$wA_{k+1} \in \mathcal{S}(\mathbb{R})$, integrating $I(\zeta)$ by parts repeatedly introduces no boundary terms. Moreover, $w'/w = -s^3/(n+s^2)$ gives the conjugation identity
$L(wG) = w\,\big[(1+s^2/n)\,\partial_s + (2s-s^3)/n\big]\,G$, so substituting
$s = n^{1/4}t$ and inducting on $J \ge 1$ yields
\begin{align*}
|I(\zeta)|
&=\Bigl|\Bigl(\frac{i}{\zeta}\Bigr)^J\int_{\mathbb R} L^J(wA_{k+1})\,e^{i\zeta\nu}\,ds\Bigr|
\overset{\eqref{eq:w-bound}}{\lesssim_J} n^{-J/4}\sum_{j=0}^J n^{(j-J)/4}\|A_{k+1}^{(j)}\|_1\\
&\overset{\eqref{eq:H2}}{\lesssim_J} n^{-J/4}(k+1)^{(J+1)/2}
\lesssim_{c,J} n^{1/4-\eta(J+1)/2}.
\end{align*}
Taking $J=\lceil1/(2\eta)\rceil$ makes the exponent negative, and combining the bounds for
$|\zeta|\le n^{1/4}$ and $|\zeta|>n^{1/4}$ finishes the proof.
\end{proof}

\subsection{Proof of Lemma \ref{lem:hd}}

Throughout this subsection, $\psi:=\psi_{k,n}=\sum_{d=0}^n b_{d,n}\Psi_{d,n}$
and $\psi^{(u)}(x):=\psi(u\odot x)$ for $u\in\{\pm1\}^n$. Since $\psi$ is odd
and orthogonal to $\mathcal P_{\le 2k}$, we have $b_{d,n}=0$ unless $d$ is odd
and $d\ge 2k+1$; moreover $\sum_d b_{d,n}^2=\|\psi\|_2^2\le1$. We need two
auxiliary facts, namely a geometric decay of the coefficients $b_{d,n}$ and a simple
bound on Krawtchouk polynomials.

\begin{lemma}\label{lem:asymp}
Let $k \ge 1$ and $n \ge 8k+11$. Then for every odd \(d\) with
\(8k+9\le d\le n-2\),
\[
\left|\frac{b_{d+2,n}}{b_{d,n}}\right| \le \frac{2}{3}.
\]
\end{lemma}

\begin{proof}
From \eqref{eq:wite}, we get
\[
b_{d,n} = (-1)^{(d-1)/2}\, \frac{1}{2\|\widetilde\Omega_{k,n}\|_\infty}\,
\binom nd^{1/2} n^{-d/2}\, \binom{(d+1)/2}{k+1}\, \Gamma\!\left(\tfrac{d}{2}+1\right)
\]
for odd $d\ge 2k+1$. A direct computation shows
\begin{align*}
\frac{|b_{d+2,n}|}{|b_{d,n}|}
&= \frac{d+3}{d+1-2k} \cdot \frac{1}{2}\sqrt{\frac{d+2}{d+1}} \cdot
   \sqrt{\left(1-\frac{d}{n}\right)\!\left(1-\frac{d+1}{n}\right)} \\
&\le \frac{8k+12}{6k+10} \cdot \frac{1}{2}\sqrt{\frac{8k+11}{8k+10}}
 \;\le\; \frac{2}{3}
\end{align*}
for $d \ge 8k+9$.
\end{proof}

\begin{lemma}\label{lem:kraw-simple}
Let $n\ge4$ and let $D$ be an integer with $1\le D\le\sqrt n$. Then for
every $0\le j\le D$ and every $r\in\{-1,-1+2/n,\dots,1\}$,
\[
|\mathcal K_j^{(n)}(r)|\le 2\,\bigl(|r|+\sqrt{D/n}\bigr)^j.
\]
\end{lemma}

\begin{proof}
Let $F_j:=\frac{n^{\underline j}}{n^j}\mathcal K_j^{(n)}(r)$, where
$n^{\underline j}:=n(n-1)\cdots(n-j+1)$, so that $F_0=1$, $F_1=r$ and, by \eqref{eq:A.30},
\[
F_j=rF_{j-1}-\frac{(j-1)(n-j+2)}{n^2}F_{j-2}\qquad(j\ge2),
\]
where the last coefficient is at most $D/n$ for $j\le D$. With
$a:=|r|+\sqrt{D/n}$ we have $|r|a+D/n\le a^2$, and induction gives
$|F_j|\le|r|a^{j-1}+\frac Dn a^{j-2}\le a^j$. Finally,
$\frac{n^{\underline j}}{n^j}=\prod_{i<j}(1-\frac in)\ge1-\frac{j^2}{2n}\ge\frac12$ as
$j\le\sqrt n$, hence $|\mathcal K_j^{(n)}(r)|\le2a^j$.
\end{proof}

\begin{lemma}[Hidden-direction family]\label{lem:hidden-direction}
Fix $0<\rho<1$, $c_0>0$, and $0<\eta<1/2$. There exist
$n_0=n_0(c_0,\eta)$ and $K_{c_0,\eta}>0$ such that for every odd
$n\ge n_0$ and every integer $1\le m\le c_0 n^{1/2-\eta}$, with
$k:=\lfloor(m+1)/2\rfloor$, the following holds for the distributions
\[
\mathcal{D}_u(x,y) := 2^{-(n+1)}\bigl(1 + y\,\psi^{(u)}(x)\bigr),
\quad
\mathcal{D}_0(x,y) := 2^{-(n+1)}
\]
on $\{\pm1\}^n\times\{\pm1\}$, where $u\in\{\pm1\}^n$.
\begin{enumerate}[label=(\roman*)]
\item For every $u\in\{\pm1\}^n$, the smoothed rotated majority
      $f_u(x) := T_\rho\operatorname{sign}(\langle u,x\rangle)$ satisfies
      \[
      \mathbb{E}_{\mathcal{D}_u}[Y f_u(X)] \ge \frac{K_{c_0,\eta}\,\rho^{2m+1}}{\sqrt{m}}.
      \]
\item For all $u\ne v$ in $\{\pm1\}^n$ with $|\langle u,v\rangle|\le n^{3/4}$,
      \[
      \chi_{\mathcal D_0}(\mathcal D_u,\mathcal D_v)
      \le\gamma_n:=4\,\bigl(2n^{-1/4}\bigr)^{2k+1},
      \]
      and
      \begin{equation}
      n^{-O(k)}\;\le\;\gamma_n\;\le\;4\,n^{-(2k+1)/8}.
      \tag{$\gamma$}\label{eq:gamma-range}
      \end{equation}
\item There exists $\mathcal{U}\subseteq\{\pm1\}^n$ with
      $|\mathcal{U}|=2^{\Omega(\sqrt n)}$ such that
      $|\langle u,v\rangle|\le n^{3/4}$ for all distinct $u,v\in\mathcal{U}$.
\end{enumerate}
\end{lemma}

\begin{proof}
Since $\|\psi\|_\infty\le1$, each $\mathcal D_u$ is a probability
distribution. Take $n_0=n_0(c_0,\eta)$ large enough so that
Lemmas~\ref{lem:dual-witness} and~\ref{lem:asymp} and the estimates
below apply uniformly over all allowed $m$. This is possible
since
\(
k\log n\le c_0n^{1/2-\eta}\log n=o(\sqrt n).
\)

Let $R_u:=\mathcal D_u/\mathcal D_0-1$, so that $R_u(x,y)=y\,\psi^{(u)}(x)$ and $\|R_u\|_{L^2(\mathcal D_0)}\le1$.

(i) Under $\mathcal D_u$ we have $X\sim U_n$. Both $f_u$ and
$\mathbb E[Y\mid X=x]=\psi^{(u)}(x)$ are the result of applying the sign flip $x\mapsto u\odot x$ to $T_\rho\operatorname{Maj}_n$
and $\psi$, since
$T_\rho$ commutes with the flip. As the $X$-marginal $U_n$ is preserved under the rotation, so is the correlation
$\mathbb E_{\mathcal D_u}[Yf_u(X)]
=\langle T_\rho\operatorname{Maj}_n,\psi\rangle_{U_n}$, which allows us to use the bound
from the proof of Lemma~\ref{lem:dual-witness} together with Lemma~\ref{lem:linfty}.

(ii) By assumption, we have $|\langle u, v\rangle|/n \le n^{-1/4}$, so 
\[
\begin{aligned}
\bigl|\langle R_u,R_v\rangle_{L^2(\mathcal D_0)}\bigr|
&= \Bigl|\sum_{d \ge 1} \langle \psi_{d}^{(u)},\psi_{d}^{(v)}\rangle\Bigr|
\overset{\mathclap{\eqref{eq:A.32}}}{\le}
   \sum_{\substack{d\ge1\\ d\ \mathrm{odd}}}  b_{d,n}^2\,|\mathcal K_d^{(n)}(\langle u,v\rangle/n)|\\
&\overset{\mathclap{\text{(Lem.~\ref{lem:kraw-simple})}}}{\le}
  2\bigl(2n^{-1/4}\bigr)^{2k+1}
  +\sum_{\substack{d>\sqrt n\\ d\ \mathrm{odd}}} b_{d,n}^2\\
&\overset{\mathclap{\text{(Lem.~\ref{lem:asymp})}}}{\le}
  2\bigl(2n^{-1/4}\bigr)^{2k+1}+e^{-\Omega(\sqrt n)}
\\ &\le\;4\bigl(2n^{-1/4}\bigr)^{2k+1}
\end{aligned}
\]
for $n$ large enough, where $\psi_d := {b}_{d,n} \Psi_{d,n}$.
The upper bound in \eqref{eq:gamma-range} follows from
$2n^{-1/4}\le n^{-1/8}$ for $n\ge256$; the lower bound is immediate.

(iii) Let $s:=\lfloor\exp(0.2\sqrt n)\rfloor$; if $s\le1$ there
is nothing to show. Otherwise, sample $u^{(1)},\dots,u^{(s)}$ independently and
uniformly from $\{\pm1\}^n$. For $a\ne b$, $\langle u^{(a)},u^{(b)}\rangle$
is a sum of $n$ independent Rademacher variables, so by Hoeffding's
inequality and a union bound,
\[
\Pr\bigl(\exists a\neq b:\ |\langle u^{(a)},u^{(b)}\rangle|>n^{3/4}\bigr)
\le s^2e^{-\sqrt n/2}\le e^{(2/5-1/2)\sqrt n}<1 .
\]
Hence some realization has all pairwise inner products with absolute value at most $n^{3/4}$; in particular, the vectors are distinct.
\end{proof}

\subsection{Proof of Theorems 4.3 and 1.3}

\begin{theorem}[Hidden-direction SQ lower bound]\label{thm:hidden-direction-SQ}
For every $c_0>0$ and $\eta_0\in(0,1/2)$, there is $n_0=n_0(c_0,\eta_0)$ such that the
following holds. Let $n\ge n_0$ be odd, $1\le m\le c_0n^{1/2-\eta_0}$ and
$k:=\lfloor(m+1)/2\rfloor$, and let $\mathcal U$ and $\gamma_n$ be as in
Lemma~\ref{lem:hidden-direction}. For $\varepsilon>0$ with $\varepsilon^2\ge2\gamma_n$,
let $P$ be the search problem on $\{\mathcal D_u\}_{u\in\mathcal U}$ whose valid
solutions for $\mathcal D_u$ are the hypotheses $h:\{\pm1\}^n\to[-1,1]$ with
$\mathbb E_{\mathcal D_u}[Yh(X)]\ge\varepsilon$, and assume that every $\mathcal D_u$
admits a valid solution.

Then any randomized SQ algorithm that solves $P$ with probability at least $2/3$ on
every $\mathcal D_u$ needs at least
\[
\tfrac23\,|\mathcal U|\gamma_n-1\;=\;2^{\Omega(\sqrt n)}
\]
queries to $\mathrm{VSTAT}(1/(6\gamma_n))$. In particular, it must either make
$2^{\Omega(\sqrt n)}$ STAT queries or use a query of tolerance
\[\tau<\sqrt{6\gamma_n}\le2\sqrt6\,n^{-(2k+1)/16}.\]
\end{theorem}

\begin{proof}
Write $\gamma:=\gamma_n$. By Lemma~\ref{lem:hidden-direction}(ii),
every nonempty $B\subseteq\mathcal U$ satisfies
\begin{equation}\label{eq:avg-corr}
\gamma(B,\mathcal D_0)
\le \gamma+\frac{1}{|B|}.
\end{equation}

We first show that any fixed hypothesis is valid for at most
$1/\gamma$ distributions. Fix $h:\{\pm1\}^n\to[-1,1]$, and let
\(
S_h:=\{u\in\mathcal U\mid h\in P_{\mathcal D_u}\}.
\)
There is nothing to prove if $S_h=\emptyset$.
Otherwise, under $\mathcal D_0$, the label
is uniform and independent of $X$, so
$\mathbb E_{\mathcal D_0}[yh(x)]=0$ and
$\|yh(x)\|_{\mathcal D_0}\le1$.
For every $u\in S_h$, a change of measure gives
\(
\langle yh,R_u\rangle_{\mathcal D_0}
=\mathbb E_{\mathcal D_u}[Yh(X)]
\ge\varepsilon.
\)
Averaging over $S_h$, applying the Cauchy--Schwarz inequality, and using
\eqref{eq:avg-corr}, we obtain
\[
\varepsilon^2
\le
\Bigl\|\frac{1}{|S_h|}\sum_{u\in S_h}R_u\Bigr\|_{\mathcal D_0}^2
\le \gamma(S_h,\mathcal D_0)
\le \gamma+\frac{1}{|S_h|}.
\]
Since $\varepsilon^2\ge2\gamma$, it follows that
$|S_h|\le1/\gamma$.

For sufficiently large $n$, set
\(
\beta:=\frac{1}{|\mathcal U|\gamma},
\)
and \(
d:=|\mathcal U|\gamma-1=(1-\beta)|\mathcal U|\gamma.
\)
For every hypothesis $h$, the subfamily
\(
\mathcal D_h:=\{\mathcal D_u \mid u\in\mathcal U\setminus S_h\}
\)
has size at least $|\mathcal U|-1/\gamma=(1-\beta)|\mathcal U|$. Moreover,
\[
\frac{|\mathcal D_h|}{d}
\ge
\frac{|\mathcal U|-1/\gamma}{|\mathcal U|\gamma-1}
=\frac{1}{\gamma}.
\]
Thus every $\mathcal A\subseteq\mathcal D_h$ with
$|\mathcal A|\ge|\mathcal D_h|/d$ has average correlation at most
$2\gamma$, by \eqref{eq:avg-corr}. By the definition of SDA,
\[
\mathrm{SDA}(P,2\gamma,\beta)\ge d.
\]

The problem $P$ is total by assumption. Hence
Theorem~\ref{thm:sq-lb} gives the lower bound
\[
\frac{2/3-\beta}{1-\beta}\,d
=\frac23|\mathcal U|\gamma-1
\overset{{\text{Lem.~\ref{lem:hidden-direction}(iii)},\, \eqref{eq:gamma-range}}}{=}
\exp\bigl(\Omega(\sqrt n)-O(k\log n)\bigr)
=2^{\Omega(\sqrt n)},
\]
on the number of queries to $\mathrm{VSTAT}(1/(6\gamma))$ for any SQ algorithm that
solves $P$ with probability at least $2/3$, where the last step uses
$k=O(n^{1/2-\eta_0})$.
This also justifies the choice $\beta<2/3$ for sufficiently large $n$. The conversion
from VSTAT to STAT is standard.
\end{proof}

\begin{theorem}\label{cor:smoothed-agnostic-sq}
Let $c_0>0$ and $\eta_0\in(0,1/2)$. There exist constants $a_0,n_0>0$, depending only
on $c_0$ and $\eta_0$, such that the following holds. Let $\sigma\in(0,0.499]$,
$\varepsilon\in(0,1/4]$ and let $n\ge n_0$ be odd. Set
\[
m:=\Bigl\lfloor a_0\,\frac{\log(1+\sigma/\varepsilon^2)}{\sigma}\Bigr\rfloor
\]
and assume $1\le m\le c_0n^{1/2-\eta_0}$.\footnote{For $\varepsilon$ below a sufficiently
small absolute constant, the restriction $m \ge 1$ can be dropped: if $m = 0$, then $\log(1+\sigma/\varepsilon^2)/\sigma < 1/a_0 = O(1)$, and a similar family of hard instances (e.g., rotated majority functions along nearly orthogonal directions) yields the same SDA lower bound. Since the upper bound on $\varepsilon$ is in this case only scaled by a constant, we ignore it for
simplicity.}
Let $\mathcal U$ and $\{\mathcal D_u\}_{u\in\mathcal U}$ be as in
Lemma~\ref{lem:hidden-direction} for this $m$. Then any randomized SQ algorithm that,
given access to an unknown $\mathcal D_u$ with $u\in\mathcal U$, outputs with probability
at least $2/3$ a hypothesis $h:\{\pm1\}^n\to\{\pm1\}$ with
$\err_{\mathcal D_u}(h)\le\opt_{\sigma,\mathcal D_u}+\varepsilon$ must either make
$2^{\Omega(\sqrt n)}$ STAT queries or use a query of tolerance
\[
\tau<2\sqrt6\,n^{-m/16}
=n^{-\Omega\left(\log(1+\sigma/\varepsilon^2)/\sigma\right)}.
\]
\end{theorem}

\begin{proof}
We apply Theorem~\ref{thm:hidden-direction-SQ} to the search problem $P$
at threshold $2\varepsilon$. To this end, we show that $T_\rho g_u$ is a valid solution, $\gamma_n\le2\varepsilon^2$, and that every
$\varepsilon$-optimal hypothesis solves $P$.

Choose $a_0>0$ sufficiently small and let
$\rho:=1-2\sigma$. Since $-\log\rho=O(\sigma)$
and $m\ge1$, the definition of $m$ gives
\[
\varepsilon^2
\le \frac{\sigma}{e^{\sigma m/a_0}-1}
\le \frac{a_0}{m}e^{-\sigma m/(2a_0)}
\le \frac{a_0\rho^{4m+2}}{m},
\]
where the middle inequality uses $e^s-1\ge se^{s/2}$ for $s\ge0$.
Thus Lemma~\ref{lem:hidden-direction}(i) gives
$\mathbb E_{\mathcal D_u}[YT_\rho g_u(X)]\ge4\varepsilon$
for the halfspace $g_u(x):=\operatorname{sign}\langle u,x\rangle$.

To bound $\gamma_n$, note that by concavity of
$s\mapsto\log(1+s/\varepsilon^2)$ we have
\[
\log\frac1\varepsilon
\le \frac{\log(1+\sigma/\varepsilon^2)}{2\sigma}
< \frac{m+1}{2a_0}
\le \frac{m}{a_0}.
\]
Hence \eqref{eq:gamma-range} implies
$\gamma_n\le4n^{-m/8}\le2\varepsilon^2$
for all sufficiently large $n$.

The search problem $P$ at threshold $2\varepsilon$ is therefore
total, with valid solutions $T_\rho g_u$. Moreover, every Boolean
hypothesis $h$ with
$\err_{\mathcal D_u}(h)\le
\opt_{\sigma,\mathcal D_u}+\varepsilon$ solves $P$, since
\[
\mathbb E_{\mathcal D_u}[Yh(X)]
=1-2\err_{\mathcal D_u}(h)
\ge \mathbb E_{\mathcal D_u}[YT_\rho g_u(X)]-2\varepsilon
\ge2\varepsilon.
\]
Thus an SQ learner for the smoothed agnostic problem provides a solution to $P$ with the same success probability
and no additional queries. Applying
Theorem~\ref{thm:hidden-direction-SQ} at threshold $2\varepsilon$
shows that any such learner with success probability at least
$2/3$ must either make $2^{\Omega(\sqrt n)}$ STAT queries
or use a query of tolerance
\[
\tau<\sqrt{6\gamma_n}
=O(n^{-m/16})
=n^{-\Omega(\log(1+\sigma/\varepsilon^2)/\sigma)},
\]
where the last step follows from the definition of $m$ and $m\ge1$.
\end{proof}

We prove the more general version of Theorem \ref{thm:smoothed-upperbound} for product distributions.

\begin{theorem}[Upper bound for product marginals]\label{thm:upper-product}
Let $\mathcal C$ be any class of Boolean functions on $\{\pm1\}^n$, and let $\varepsilon>0$, $\sigma\in(0,1/2)$ and $\delta\in(0,1)$. There is an algorithm that, given samples from any distribution $\mathcal D$ on $\{\pm1\}^n\times\{\pm1\}$ whose $X$-marginal $\mu$ is a product distribution, has sample complexity and running time $\mathrm{poly}(n^{\log(1/\varepsilon)/\sigma},\log(1/\delta))$ and outputs a Boolean hypothesis $h$ with $\err(h)\le\opt_{\mathcal C,\sigma}+\varepsilon$ with probability at least $1-\delta$.
\end{theorem}

\begin{proof}
We use the $L^1$ polynomial regression algorithm~\citep[Theorem~5]{KKMS08}. For this, it suffices to show that for some $d = O(\log(1/\varepsilon)/\sigma)$, every $f\in\mathcal C$ admits a polynomial $p_f\in \mathcal{P}_{\le d}$ with $\|T_\rho f- p_f\|_{L^1(\mu)} < \varepsilon/2$, where $\rho := 1-2\sigma$. We show this by taking $p_f$ as the truncation of $T_\rho f$ at degree $d$ in the $\mu$-biased Fourier basis and bounding the $L^2(\mu)$ error.

Let $m_i := \mathbb E_\mu[X_i]$. Coordinates with $|m_i|=1$ can be dropped after averaging $f$ over their noisy values $m_iZ_i$, where $Z \sim B(\sigma)^n$. This preserves $T_\rho f$ on the support of $\mu$ and yields a function in $[-1,1]$, which suffices below. Relabeling the remaining coordinates and denoting the resulting function again by $f$, assume $|m_i|<1$ for all $i$. The law $\nu$ of $X\odot Z$ is again a product distribution, with means $\rho m_i$. Let $\{\phi_S\}$ and $\{\psi_S\}$ be the biased Fourier bases of $L^2(\mu)$ and $L^2(\nu)$~\citep[Chapter~8]{ODonnell14}, and write $f = \sum_{S\subseteq[n]} \tilde f(S)\psi_S$. Since $T_\rho$ acts coordinatewise and preserves means, $T_\rho\psi_S=\lambda_S\phi_S$
with $\lambda_S:=\prod_{i\in S}\lambda_i$, where
\[
\lambda_i=\operatorname{Corr}(X_i,X_iZ_i)=\rho\sqrt{\frac{1-m_i^2}{1-\rho^2m_i^2}}\le\rho .
\]
Hence $T_\rho f=\sum_{S\subseteq[n]}\lambda_S\tilde f(S)\phi_S$. Setting
$p_f:=\sum_{|S|\le d}\lambda_S\tilde f(S)\phi_S$, we obtain
\[
\|T_\rho f-p_f\|_{L^2(\mu)}^2
=\sum_{|S|>d}\lambda_S^2\,\tilde f(S)^2
\le\rho^{2(d+1)}\sum_{S\subseteq[n]}\tilde f(S)^2
=\rho^{2(d+1)}\|f\|_{L^2(\nu)}^2
\le\rho^{2(d+1)}.
\]
It holds $\rho = 1-2\sigma \le e^{-2\sigma}$, so taking $d$ at least a sufficiently large constant multiple of $\log(1/\varepsilon)/\sigma$ gives $\|T_\rho f - p_f\|_{L^1(\mu)} \le \|T_\rho f - p_f\|_{L^2(\mu)} < \varepsilon/2$. As $\mathcal{P}_{\le d}$ has dimension $n^{O(d)}$, the algorithm runs in $\mathrm{poly}(n^{\log(1/\varepsilon)/\sigma},\log(1/\delta))$ time and sample complexity.
\end{proof}


\newpage
\section*{NeurIPS Paper Checklist}

\begin{enumerate}

\item {\bf Claims}
    \item[] Question: Do the main claims made in the abstract and introduction accurately reflect the paper's contributions and scope?
    \item[] Answer: \answerYes{} 
    \item[] Justification: The abstract and introduction state the SQ lower bound (Theorem~\ref{cor:smoothed-agnostic-sq}) and the matching upper bound (Theorem~\ref{thm:smoothed-upperbound}), with the precise scope (uniform marginal on the Boolean cube) and the technical contribution (sine-Laguerre witness, Lemma~\ref{lem:dual-witness}) clearly identified.

\item {\bf Limitations}
    \item[] Question: Does the paper discuss the limitations of the work performed by the authors?
    \item[] Answer: \answerYes{}{} 
    \item[] Justification: The conclusion of this work discusses the remaining logarithmic gap between upper and lower bounds, the restriction to the uniform marginal, the remaining gap on the upper bound side in the subexponential setting of \citet{KM25}, and the open question of whether non-SQ algorithms can bypass the barrier. 

\item {\bf Theory assumptions and proofs}
    \item[] Question: For each theoretical result, does the paper provide the full set of assumptions and a complete (and correct) proof?
    \item[] Answer: \answerYes{} 
    \item[] Justification: All theorems and lemmas in this work are numbered and cross-referenced. Informal statements in the main text (Theorems~1.3, 1.4, Lemmas~4.1, 4.2, Theorem~4.3) are paired with formal versions in the appendix (Corollary~\ref{cor:smoothed-agnostic-sq}, Theorem~\ref{thm:smoothed-upperbound}, Lemmas~\ref{lem:dual-witness}, \ref{lem:hidden-direction}, Theorem~\ref{thm:hidden-direction-SQ}) with complete proofs. All assumptions are stated explicitly in the formal theorem statements.

    \item {\bf Experimental result reproducibility}
    \item[] Question: Does the paper fully disclose all the information needed to reproduce the main experimental results of the paper to the extent that it affects the main claims and/or conclusions of the paper (regardless of whether the code and data are provided or not)?
    \item[] Answer: \answerNA{} 
    \item[] Justification: This work does not include experiments.

\item {\bf Open access to data and code}
    \item[] Question: Does the paper provide open access to the data and code, with sufficient instructions to faithfully reproduce the main experimental results, as described in supplemental material?
     \item[] Answer: \answerNA{} 
    \item[] Justification: This work does not include experiments.

\item {\bf Experimental setting/details}
    \item[] Question: Does the paper specify all the training and test details (e.g., data splits, hyperparameters, how they were chosen, type of optimizer) necessary to understand the results?
     \item[] Answer: \answerNA{} 
    \item[] Justification: The paper does not include experiments.

\item {\bf Experiment statistical significance}
    \item[] Question: Does the paper report error bars suitably and correctly defined or other appropriate information about the statistical significance of the experiments?
     \item[] Answer: \answerNA{} 
    \item[] Justification: The paper does not include experiments.

\item {\bf Experiments compute resources}
    \item[] Question: For each experiment, does the paper provide sufficient information on the computer resources (type of compute workers, memory, time of execution) needed to reproduce the experiments?
     \item[] Answer: \answerNA{} 
    \item[] Justification: This work does not include experiments.
    
\item {\bf Code of ethics}
    \item[] Question: Does the research conducted in the paper conform, in every respect, with the NeurIPS Code of Ethics \url{https://neurips.cc/public/EthicsGuidelines}?
    \item[] Answer: \answerYes{} 
    \item[] Justification: This work does not involve any human subjects, datasets, or deployed systems.

\item {\bf Broader impacts}
    \item[] Question: Does the paper discuss both potential positive societal impacts and negative societal impacts of the work performed?
    \item[] Answer: \answerNA{} 
    \item[] Justification: There is no direct path to any specific societal applications from the theoretical results presented in this work. 
    
\item {\bf Safeguards}
    \item[] Question: Does the paper describe safeguards that have been put in place for responsible release of data or models that have a high risk for misuse (e.g., pre-trained language models, image generators, or scraped datasets)?
    \item[] Answer: \answerNA{} 
    \item[] Justification: This work does not release any models, datasets or other artifacts that pose the risk of misuse. \

\item {\bf Licenses for existing assets}
    \item[] Question: Are the creators or original owners of assets (e.g., code, data, models), used in the paper, properly credited and are the license and terms of use explicitly mentioned and properly respected?
    \item[] Answer: \answerNA{} 
    \item[] Justification: This work does not use existing assets.

\item {\bf New assets}
    \item[] Question: Are new assets introduced in the paper well documented and is the documentation provided alongside the assets?
    \item[] Answer: \answerNA{} 
    \item[] Justification: This work does not introduce any new assets.

\item {\bf Crowdsourcing and research with human subjects}
    \item[] Question: For crowdsourcing experiments and research with human subjects, does the paper include the full text of instructions given to participants and screenshots, if applicable, as well as details about compensation (if any)? 
    \item[] Answer: \answerNA{} 
    \item[] Justification: This work does not involve crowdsourcing or research with human subjects.

\item {\bf Institutional review board (IRB) approvals or equivalent for research with human subjects}
    \item[] Question: Does the paper describe potential risks incurred by study participants, whether such risks were disclosed to the subjects, and whether Institutional Review Board (IRB) approvals (or an equivalent approval/review based on the requirements of your country or institution) were obtained?
    \item[] Answer: \answerNA{} 
    \item[] Justification: This work does not involve crowdsourcing or research with human subjects.

\item {\bf Declaration of LLM usage}
    \item[] Question: Does the paper describe the usage of LLMs if it is an important, original, or non-standard component of the core methods in this research? Note that if the LLM is used only for writing, editing, or formatting purposes and does \emph{not} impact the core methodology, scientific rigor, or originality of the research, declaration is not required.
    \item[] Answer: \answerNA{}
    \item[] Justification: This work does not involve LLMs as any important, original, or non-standard components.

\end{enumerate}

\end{document}